\documentclass[accepted]{article}
\usepackage{icml2025}
\usepackage{amsthm}
\usepackage{amsmath,amssymb,mathtools, thm-restate}

\usepackage{enumitem} 
\usepackage{bm}
\usepackage[most]{tcolorbox}
\usepackage{tabularx}
\usepackage{natbib}
\tcbuselibrary{breakable,skins}
\usepackage{fancyhdr} 
\usepackage{hyperref}
\usepackage[noabbrev]{cleveref}
\usepackage{multirow, booktabs} 
\usepackage{todonotes}
\usepackage{subcaption} 

\usepackage{listings}\usepackage{algorithm}
\usepackage{algorithmic}

\usepackage{authblk}

\newcommand{\X}{{\cal X}}

\newcommand{\Z}{{\cal Z}}

\newcommand{\B}{{\cal B}}

\newcommand{\F}{{\cal F}}
\newcommand{\Y}{{\cal Y}}

\renewcommand{\S}{{\cal S}}
\newcommand{\Score}{\ensuremath{\textrm{Score}}}

\newcommand{\x}{{\bm{x}}}

\newcommand{\z}{{\bm{z}}}

\newcommand{\truerisk}{\ensuremath{R_{\textrm{0/1}}}}
\newcommand{\Nesyrisk}{\ensuremath{R_{\textrm{NeSy}}}}
\newcommand{\empNesyrisk}{\ensuremath{\hat{R}_{\textrm{NeSy}}}}
\newcommand{\PNLrisk}{\ensuremath{R_{\textrm{PNL}}}}

\newcommand{\AAABLrisk}{\ensuremath{R_{\textrm{A}^3}}}
\newcommand{\ABLrisk}{\ensuremath{R_{\textrm{ABL}}}}

\newcommand{\KB}{\ensuremath{\mathtt{KB}}}

\newcommand{\Exp}{{\mathbb E}}  
\newcommand{\Ind}{{\mathbb I}}  

\newcommand{\AAABL}{\ensuremath{\text{A}^3\text{BL}}}
\newcommand{\nesytask}{\ensuremath{\mathcal{T}}}
\newcommand{\ambiguityDegree}{DCSP solution disagreement}

\newcommand*{\img}[1]{%
  \raisebox{-.2\baselineskip}{%
    \includegraphics[height=0.8\baselineskip, width=0.8\baselineskip, keepaspectratio]{#1}%
  }%
}
\newcommand{\IMGO}{{\img{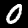}}}
\newcommand{\IMGI}{{\img{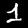}}}
\newcommand{\IMGII}{{\img{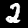}}}

\newcommand{\IMGNine}{{\img{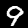}}}

\definecolor{lightred}{RGB}{216,144,144}
\definecolor{slightred}{RGB}{251,243,243}
\definecolor{slightgreen}{RGB}{243,248,243}
\definecolor{slightblue}{RGB}{28,130,185}
\definecolor{sgreen}{RGB}{0, 128, 0}

\makeatletter
\newcommand{\subalign}[1]{%
  \vcenter{%
    \Let@ \restore@math@cr \default@tag
    \baselineskip\fontdimen10 \scriptfont\tw@
    \advance\baselineskip\fontdimen12 \scriptfont\tw@
    \lineskip\thr@@\fontdimen8 \scriptfont\thr@@
    \lineskiplimit\lineskip
    \ialign{\hfil$\m@th\scriptstyle##$&$\m@th\scriptstyle{}##$\hfil\crcr
      #1\crcr
    }%
  }%
}
\makeatother

\DeclareDocumentCommand{\citealtt}{o m}{%
  \IfNoValueTF{#1} 
  {\citeauthor{#2}~(\citeyear{#2})} 
  {#1~(\citeauthor{#2},~\citeyear{#2})} 
}

\crefname{equation}{}{} 

\theoremstyle{plain}
\newtheorem{theorem}{Theorem}[section]

\newtheorem{lemma}[theorem]{Lemma}
\newtheorem{corollary}[theorem]{Corollary}

\theoremstyle{definition} 
\newtheorem{definition}[theorem]{Definition}
\newtheorem{example}{Example}

  \definecolor{mydarkblue}{rgb}{0,0.08,0.45}
  \hypersetup{ %
    pdftitle={},
    pdfsubject={Proceedings of the International Conference on Machine Learning 2025},
    pdfkeywords={},
    pdfborder=0 0 0,
    pdfpagemode=UseNone,
    colorlinks=true,
    linkcolor=mydarkblue,
    citecolor=mydarkblue,
    filecolor=mydarkblue,
    urlcolor=mydarkblue,
  }
  
\definecolor{codegray}{gray}{0.98}
\definecolor{codepurple}{RGB}{128,0,128}
\definecolor{codegreen}{rgb}{0,0.6,0}

\lstdefinelanguage{Python}{
    morekeywords={class, def, for, while, if, else, return, break, continue, import, from, as, try, except, finally, raise, with, lambda, in, is, and, or, not, None, True, False},
    sensitive=true,
    morecomment=[l]{\#},
    morestring=[b]",
    morestring=[b]',
}

\lstdefinestyle{mystyle}{
    commentstyle=\color{codegreen},
    keywordstyle=\color{magenta},
    numberstyle=\tiny\color{codegray},
    stringstyle=\color{codepurple},
    basicstyle=\ttfamily\footnotesize,
    breakatwhitespace=false,         
    breaklines=true,                 
    captionpos=b,                    
    keepspaces=true,                 
    numbersep=5pt,                  
    showspaces=false,                
    showstringspaces=false,
    showtabs=false,                  
    tabsize=2
}

\begin{document}

\icmltitle{{A Learnability Analysis on Neuro-Symbolic Learning}}


\icmlsetsymbol{equal}{*}

\begin{center}
  \textbf{Hao-Yuan He}, \textbf{Ming Li}\\
  \vspace{3mm}
  National Key Laboratory for Novel Software Technology, Nanjing University, China\\
  School of Artificial Intelligence, Nanjing University, China\\
  \{hehy,lim\}@lamda.nju.edu.cn
\end{center}


\icmlcorrespondingauthor{Ming Li}{lim@lamda.nju.edu.cn}
\icmlkeywords{Machine Learning, ICML}

\vskip 0.3in

\printAffiliationsAndNotice{}  
\hyphenation{minimal Abductive neural satisfied analyses supervision machine learning Neuro-Symbolic pseudo-label end-to-end}




\begin{abstract}
    This paper analyzes the learnability of neuro-symbolic~(NeSy) tasks in hybrid systems, such as probabilistic NeSy methods and abductive learning. We demonstrate that the learnability of NeSy tasks can be characterized by their derived constraint satisfaction problems (DCSPs). Specifically, a task is \emph{learnable} if the corresponding DCSP has a unique solution; otherwise, it is \emph{unlearnable}. For learnable tasks, we derive the corresponding sample complexity. For general tasks, we establish the asymptotic error and demonstrate that the expected error scales with the disagreement among solutions. Our results offer a principled approach to determining learnability and provide insights for new algorithms.
\end{abstract}

\section{Introduction}\label{sec:intro}
Neuro-symbolic learning~(NeSy) aims to integrate data-driven learning with knowledge-driven reasoning into a unified framework~\citep{nesySurvey:2022,nesySurvey:aij:2024}.
Current state-of-the-art NeSy methods primarily adopt a hybrid approach, combining learning and reasoning systems, such as ABL~\citep{zhou2019abductive,dai_abl_2019} and DeepProbLog~\citep{manhaeve_deepproblog_2018, deepproblog_aij}.

An illustrative prototype is shown in \cref{fig: overview of NeSy}.
Initially, the system employs a learning model, to map input queries $\x$ to corresponding concepts $\hat\z$. 
Then, a symbolic system~(\KB{}), such as a first-order logic solver, processes $\hat\z$ to deduce the final answer $\hat y$. The system evaluates the predicted concepts \(\hat\z\), for instance, by verifying whether \(\hat{\z} \land \KB \models y\). Feedback from \KB{} is provided to the learning model in various forms, such as pseudo-labels in ABL or weighted model counting in DeepProbLog, to facilitate further improvements. This prototype can be widely applied in various domains, such as puzzle solving, code generation, and self-driving path planning~\citep{jiao2024text2sql_nesy,Li2024NeSyLC,hu2025ABLRefl}.

The hybrid system~(cf. \cref{fig: overview of NeSy}) enables end-to-end training in a weakly supervised manner using only $(\x, y)$ pairs, i.e., training the model \(f: \X\to \Z\) relying solely on the raw data and the final answers.
Since the concepts \(\z\) are not given, NeSy methods aim to minimize the discrepancy between the model's prediction and the knowledge base:
\begin{equation}
  \Nesyrisk(f) = \operatorname{\Exp}\limits_{(\x,y)}\left[ \Ind\left(f(\x) \land \KB \not\models y\right)\right]. \label{eq: nesyrisk intro}
\end{equation}

Note that the goal is to learn the model \(f\) to generalize well to unseen data, which requires a low concept error. 
The concept risk is defined using the labeling function \(g\):
\begin{equation}
  R_{0/1}(f) = \operatorname{\Exp}\limits_{x}\left[ \Ind\left(f(x) \neq g(x)\right)\right]. \label{eq: concept risk intro}
\end{equation}

\begin{figure}[!tb]
  \centering
  \includegraphics[width=.45\linewidth]{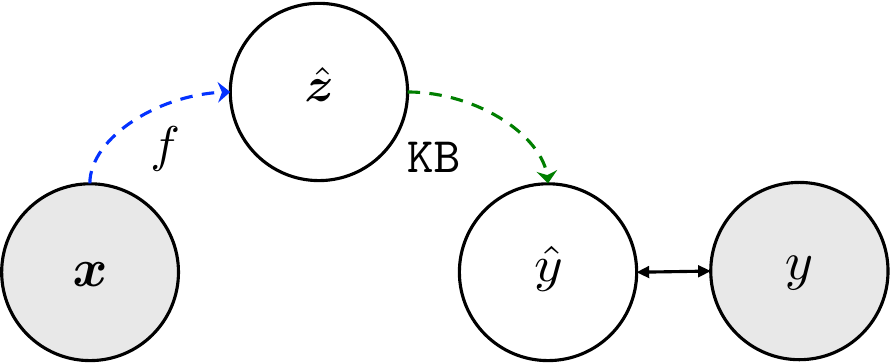}
  \caption{A typical inference process of \emph{hybrid neuro-symbolic system}.
    Shadowed circles denote observed variables, \(\x\) is raw input data, \(\hat\z\) is intermediate concepts, \(\hat y\) is the final answer inferred by \KB{}, and \(y\) denotes the true final answer.
    The goal is to learn the model \(f\).
  }
  \label{fig: overview of NeSy}
\end{figure}

In this context, we are concerned with the question of \emph{learnability}.
That is, under what conditions can the concept risk~\cref{eq: concept risk intro} be minimized via empirical risk minimization over the NeSy risk~\cref{eq: nesyrisk intro}, given a finite sample set, as it approaches infinity?
In the learnable case, minimizing \cref{eq: concept risk intro} is achievable by learning from a finite-sized dataset.
However, in the unlearnable case, it is not possible to minimize \cref{eq: concept risk intro} through learning from the dataset.

Consider an unlearnable example: $\mathtt{XOR(\IMGI,\IMGO) = 1}$, where the background knowledge specifies a logical exclusive-or digit equation.
The two hypotheses, $(\IMGO \mapsto \mathtt{0}, \IMGI \mapsto \mathtt{1})$ and $(\IMGO \mapsto \mathtt{1}, \IMGI \mapsto \mathtt{0})$, both satisfy the background knowledge. 
As a result, the learning model may fail to identify the correct intermediate concept, since either hypothesis satisfies the background knowledge. For this task, even an infinite amount of training data does not guarantee a model that minimizes \cref{eq: concept risk intro}.
Therefore, \emph{it is crucial to characterize which types of NeSy tasks are (un)learnable.}

In this paper, we characterize the learnability of NeSy tasks within the probably approximately correct~(PAC) framework~\citep{valiant84Learnable}.
The key to exploring the learnability of a NeSy task is to construct it as a \emph{derived constraint satisfaction problem}~(DCSP).
Below, we present an informal version of our main result.

\begin{theorem}[Informal]\label{thm: main theorem}
  For a neuro-symbolic task \(\mathcal{T}\) with a proper hypothesis space, the learnability is determined by the conditions:
  \begin{itemize}[itemsep=0.2ex,topsep=0.1ex,leftmargin=1.2em]
    \item
      If the derived constraint satisfaction problem has a unique solution, the task is \emph{learnable}.
      Specifically, the concept error is bounded by \(\epsilon\), given that the sample size \(N\) satisfies the following condition:
      \begin{equation*}
        N > \frac{1}{\kappa} \cdot \log\left({|\B|}/{\epsilon}\right)\,,
      \end{equation*}
      where \(|\B|\) denotes a task-specific constant, and \(\kappa\) is a positive constant determined by the data distribution.
    \item
      Otherwise, the task is \emph{unlearnable}.
  \end{itemize}
\end{theorem}

A formal version of this result is provided in \cref{thm: learnable}.
Using the above theorem, we can formulate the NeSy task as a constraint satisfaction problem and leverage modern solvers~\citep{Prud2022Choco} to determine whether the task is (un)learnable.
Furthermore, once the solution space of the DCSP is obtained, we can calculate the disagreement \(d\) among the solutions~(cf. \cref{sec: conditions of learnability}).
For general NeSy tasks, in the asymptotic regime of infinitely large sample sizes, the expected error from empirical risk minimization is bounded by \(d/L\), where \(L\) denotes the size of the concept space.


The remainder of this paper is structured as follows: 

In \cref{sec:previous nesy}, we review the background and principled methods of neuro-symbolic learning. Next, in \cref{sec:ambiguity of nesy}, we introduce the learnability analysis, including formal definitions, conditions for learnability, and an asymptotic error analysis. Additionally, we discuss how previously unlearnable tasks can become learnable under the DCSP framework. Furthermore, in \cref{sec:experiments}, we validate our theoretical findings using experimental results. Finally, \cref{sec: limitation} and \cref{sec: conclusion} provide a discussion on future directions and the conclusion of the paper.

\section{Related Works}\label{sec:related_works}

The combination of learning and reasoning remains the holy grail problem of AI for decades now~\citep{kbANN94,sun1994integrating,nsfa2002}.
One promising approach is to directly incorporate logical constraints into the loss function as the optimization objective~\citep{xu_semantic_2018,roychowdhury_regularizing_2021,he2024RILL}.
However, since the logical constraint is discrete, the optimization must be projected into the continuous space.
This requires an approximation of logical reasoning.
Such an approach can lead to issues when approximating discrete logical computations~\citep{van_krieken_survey_2022}.
A more effective approach is the hybrid system, where both the learning and reasoning models function at their full capacity.
For instance, DeepProbLog~\citep{manhaeve_deepproblog_2018,deepproblog_aij}, NeurASP~\citep{neurasp} and Scallop~\citep{Li2023Scallop} employ probabilistic logic programming as their reasoning model.
ABL~\citep{zhou2019abductive,dai_abl_2019} employs abductive reasoning for logical inference.
Recently, there have been some studies on NeSy with auto-regressive or temporal models~\citep{manginas2025nesya,desmet2025relationalneurosymbolicmarkovmodels}.
Building on the hybrid approach, there have been many successful applications~\citep{Mao2019NS-CL, wang2021abl, cai2021abl, verreet2023pu_nesy,gao2024abl, jiao2024text2sql_nesy}.
Since the hybrid approach has shown its superiority, it is worthwhile to establish a theoretical framework for analyzing its learnability.

Recently, the ``reasoning shortcut'' problem~\citep{marconato23ReasoningShortcuts,Marconato2023Cool, li2023learning_wo_shortcut, marconato24bears, bortolotti2024RSbench} has been observed, referring to the mismatch phenomenon between NeSy risk and concept risk. 
The underlying reason for this issue could be that the specific NeSy task is inherently unlearnable. 
The analysis of learnability in this paper may help in understanding this phenomenon and for effective algorithm design.

There are also several studies that aim to provide theoretical insights into NeSy methods. 
For instance, \citet{wang_miws_2023} propose a weakly supervised learning framework, termed multi-instance partial label learning (MI-PLL), to study NeSy learning. 
However, their approach relies on the condition that, for any \(z \in \Z\), there exists a unique \(y\) such that \(\sigma(z, z, \dots) = y\), which assumes \([z, z, \dots]\) must constitute a valid input to the symbolic system, which often does not hold in practice. 
Besides, \citet{tao2024abl} examine scenarios where randomly selecting abduction candidates results in a consistent optimization objective within the ABL framework. 
Additionally, \citet{yang24abl} introduce a shortcut risk metric to quantify the gap between true risk and surrogate risk. 
They establish an upper bound on shortcut risk based on the complexity of the knowledge base. 
However, their findings do not offer a comprehensive framework for evaluating the learnability of NeSy tasks. 
Despite prior efforts, a substantial research gap persists in understanding the learnability of NeSy tasks.

\section{Preliminaries}\label{sec:previous nesy}
Let lowercase letters, e.g., \(x, z\), to denote instances, uppercase letters, e.g., \(X, Z\), to denote random variables, and boldface letters, e.g., \(\x, \z\), to denote vectors or sequences.
Let \(P(\cdot)\) represents a distribution, while \(p(\cdot)\) or \(\Pr[\cdot]\) denotes probability, and \([m]\) stands for the set \(\{1, \ldots, m\}\).

In this section, we first set up the problem.
Then we provide a brief review of the principal methods in this field.

\subsection{Problem Setup}
A typical hybrid neuro-symbolic system consists of two components: a \emph{machine learning} model (e.g., a neural network) and a \emph{logical reasoning} model (e.g., a first-order logic solver).
The learning model \(f: \X \rightarrow \Z\) maps an instance \(x\) (e.g., image, text, or audio) from the input space \(\X\) to an intermediate concept \(z\) (e.g., primitive facts or predicates) within the symbol space \(\Z\), where \(|\Z| = L\).
The reasoning model \(\KB\) consists of rules over the concept space and can be implemented using any logic-based system, such as ProbLog~\citep{prolog2007} or answer set programming~(\citealt{ASP}).
Assume that a labeling function \(g: \X \rightarrow \Z\) exists such that \(z = g(x)\).
The learning model is parameterized by \(\theta\), and \(p_\theta(\cdot)\) represents the likelihood estimated by the model, where \(f(x) = \arg\max_{z \in \Z} p_\theta(z \mid x)\).

During the inference process~(cf. \cref{fig: overview of NeSy}), the learning model \(f\) accepts multiple instances as a sequence \(\x = (x_1, \ldots, x_m)\) and outputs a sequence of concepts \(\hat{\z} = (\hat{z}_1, \ldots, \hat{z}_m)\).
The output \(\hat\z\) is then passed to the reasoning model \KB{}, which infers the final label \(y \in \Y\) through logical entailment, i.e., \(\hat{\z} \land \KB \models y\).
To simplify the inference process of \(\KB\), we represent it by a logical forward operator \(\sigma(\cdot)\) such that \(\sigma(\hat{\z}) = y\).
In a standard neuro-symbolic learning setup~\citep{dai_abl_2019,deepproblog_aij,Li2023Scallop, nesySurvey:aij:2024}, the training data \(\{(\x_i, y_i)\}_{i=1}^N\) are sampled from data distribution \(\mathcal{D} = (\X^m, \Y)\).
Therefore, a neuro-symbolic task can be formally defined as a triple \(\nesytask = \left<\X, \Y, \KB \right>\).
\begin{example}[Addition]\label{example:sum2}
  The input $\x \in \X^2$, where $\X$ denotes digit images.
  The concepts \(\Z\) consist of digits from \(\mathtt{0}\) to \(\mathtt{9}\).
  The data takes the form \((\IMGO, \IMGI) \mapsto \mathtt{1}\),
  with the logical forward function \(\sigma(\cdot) := \mathtt{SUM}(\cdot, \cdot)\)\,.
  The label space \(\Y\) is defined as the addition results, i.e., from \(\mathtt{0}\) to \(\mathtt{18}\).
\end{example}
Note that when all final labels \(y\) are identical (e.g., \(\z \land \KB \models \top\)), as in the case of code generation where all code must satisfy a syntax constraint~\citep{jiao2024text2sql_nesy}, the final label \(y\) can be omitted for simplicity. The analysis presented in this paper can be easily adapted to such cases.

The success of NeSy systems highly depends on the recognition of intermediate concepts. To evaluate the concept-level performance, we define the concept risk as follows:

\begin{definition}
  The concept risk is defined as follows:
  \begin{equation}\label{eq:concept risk}
    R_{0/1}(f;g) = \operatorname{\Exp}\limits_{x}\left[ \Ind\left(f(x) \neq g(x)\right)\right].
  \end{equation}
\end{definition}
For simplicity, we omit \(g\) and denote \cref{eq:concept risk} as \(R_{0/1}(f)\).
However, optimizing \cref{eq:concept risk} is challenging due to the lack of supervision regarding the intermediate concept.

\subsection{Neuro-Symbolic Methods}
In order to minimize the concept risk~\cref{eq:concept risk}, the key idea of current NeSy methods~\citep{manhaeve_deepproblog_2018,dai_abl_2019} is to optimize the neuro-symbolic risk as surrogate, which aims to minimize the discrepancy between the learning and reasoning models.

\begin{definition}
  The neuro-symbolic risk is formally defined as follows:
  \begin{equation}
    \Nesyrisk(f) = \operatorname{\Exp}\limits_{(\x,y)}\left[ \Ind\left(f(\x) \land \KB \not\models y\right)\right]. \label{eq: nesy risk}
  \end{equation}
\end{definition}

The optimization process is to select the optimal function \(f^* \in \mathcal{F}\) that minimizes the NeSy risk; that is:
\begin{equation}
  f^* = \arg\min_{f\in \F}\,\Nesyrisk(f).
\end{equation}







\paragraph{Probabilistic Neuro-Symbolic Learning.}
Probabilistic neuro-symbolic learning~(PNL, \citealt{Manhaeve21PNL}) methods adopt reasoning models via probabilistic logic programming, such as DeepProblog~\citep{manhaeve_deepproblog_2018,deepproblog_aij}, NeurASP~\citep{neurasp}, and Scallop~\citep{Li2023Scallop}.
Since the NeSy risk~\cref{eq: nesy risk} is non-continuous, PNL aims to minimize the following objective:
\begin{equation}\label{eq:pnl_fml_primitive_obj}
  -\Exp_{(\x,y)}\Pr\left[y\mid \x; f,\KB\right].
\end{equation}

Reformulating \cref{eq:pnl_fml_primitive_obj}, we can express the objective as follows:
\begin{equation}
  -\Exp_{(\x,y)}
  \log \sum_{\z} \Ind(\z \land \KB \models y) \cdot \Pr\left[\z \mid \x; f, \KB\right]. \label{eq:pnl_obj}
\end{equation}
The above objective is referred to as the probabilistic neuro-symbolic learning risk, denoted as \(\PNLrisk(f)\).

The key operation for calculating the PNL risk is \( \sum_{\z}\Ind(\z\land \KB \models y) \cdot \Pr\left[\z \mid \x; f, \KB\right]\), also well-known as \emph{weighted model counting}~(WMC), which requires enumerating all possible worlds that satisfy the constraints of the symbolic system.
This operation can be performed using various approaches, such as ProbLog~\citep{prolog2007}, answer set programming~\cite{ASP} and so on.
However, in general, the computational complexity of WMC is \#P~\citep{hardnessNesy}, which makes PNL methods challenging to scale.

\paragraph{Abductive Learning.}
Unlike PNL methods, abductive learning methods~\citep{dai_abl_2019,huang2021ablsim,hu2025ABLRefl} infer the \emph{most} plausible concepts through abductive reasoning and use them to update the model.
The objective of ABL is to minimize the risk:
\begin{equation}\label{eq:abl_risk}
  \ABLrisk(f) = -\Exp_{(\x,y)} \log \left(\Pr\left[ y, \bar{\z}\mid \x;f,\KB\right]\right),
\end{equation}
where \( \bar{\z} = \min_{\z \in A(y)} \Score(\z, f(\x))\) represents the most likely candidate in the abduction set.
The abduction set \(A(y)\) includes all possible concepts \(\z\) that satisfy the constraints of \KB{} and have a non-zero measure, i.e., \(\Pr[\z] > 0\).
The score function measures the alignment between a candidate \(\z\) and the model's prediction \(f(\x)\). For instance, \citet{dai_abl_2019} use the Hamming distance~\citep{hamming_distance} as the score function.

ABL enhances computational efficiency by concentrating on the most plausible candidates, thereby avoiding the enumeration of all possible worlds.
However, the inherent ambiguity in the abduction process can lead to incorrect candidate selection~\citep{Magnani09Abductive}, introducing bias into the learning process~\citep{a3bl}.

\paragraph{Unified View.}
Here, we provide a unified view that: 
Both PNL and ABL methods can effectively optimize~\cref{eq: nesy risk}.
Formally, we state the following theorem, with the proof provided in \cref{appendix: proof of surrogate}.

\begin{restatable}{theorem}{thmsurrogate}\label{thm: surrogate}
  A minimizer of \PNLrisk{} or \ABLrisk{} is also a minimizer of \Nesyrisk{}. For each surrogate \(R_s \in \{\PNLrisk{}, \ABLrisk{}\}\), we have:
  \[
    \arg\min_{f\in\F} R_s(f) \subseteq \arg\min_{f\in \F} \Nesyrisk(f)\,.
  \]
\end{restatable}

\section{Learnability Analysis}\label{sec:ambiguity of nesy}

Recall that the learnability discussed here refers to whether the concept risk~\cref{eq: concept risk intro} can be minimized via empirical risk minimization~(ERM) over the NeSy risk~\cref{eq: nesyrisk intro}, given a finite sample set as it approaches infinity. While this may be possible in some cases, it is not guaranteed in others. To understand why this occurs, it is crucial to establish a learnability analysis to determine: \emph{which types of NeSy tasks are learnable?}
Similar to standard probably approximately correct (PAC) learning~\citep{valiant84Learnable}, we define the learnability of a NeSy task as follows:
\begin{definition}
  Let \( N\) denote the size of samples drawn i.i.d. from \(\mathcal{D}\),\, \(\mathcal{T}\) represent a NeSy task, and \( \F \) denote the hypothesis space.
  We say that \( \mathcal{T} \) is \emph{learnable} if: for any \( 0 <\epsilon,\, \delta < 1\,\) and distribution \(\mathcal{D}\), there exists an algorithm \( \mathcal{A} \) and an integer \( N_{\epsilon, \delta} \) such that, whenever \( N \geq N_{\epsilon, \delta} \), the selected hypothesis \( \hat{f} \) satisfies
  \(\Pr[\truerisk(\hat{f}) \leq \epsilon] \geq  1 - \delta.\)
  Otherwise, we say that it is \emph{unlearnable}.
\end{definition}

We focus on the ERM algorithm to be \(\mathcal{A}\) in our analysis, as it has been shown that in common learning settings, such as supervised classification and regression, a problem is learnable if and only if it is learnable by ERM~\citep{Blumer89Learnability, Alon97Learnability}.

\subsection{Restricted Hypothesis Space}
Unlike supervised learning, where the goal is to learn a hypothesis mapping \(x \mapsto y\) from the \((x, y)\) pair, the NeSy task aims to learn a mapping \(\x \mapsto \z\) from the \((\x, y)\) pair.
Ideally, each \(y\) corresponds to a unique \(\z\), ensuring that \(\forall y \in \Y, |A(y)| = 1\). Under this condition, the task simplifies to a standard learning problem~\citep{vapnik1999overview}.
However, this condition is rarely satisfied in practical applications, making NeSy tasks inherently difficult to learn.
When multiple plausible solutions \(\z_1, \dots, \z_k\) exist, the learning task becomes more complex owing to inherent ambiguity. We refer to the task as \emph{ambiguous} if there exists some \(y \in \Y\) such that \(|A(y)| \geq 2\).
\begin{restatable}{proposition}{propCoA}\label{thm:curse of ambiguity}
  For an ambiguous NeSy task \nesytask{}, if the hypothesis space \(\F\) is sufficiently complex (e.g., capable of shattering the task), there exists \(f^*\) that minimizes \Nesyrisk{} but does not minimize the \truerisk{}.
\end{restatable}

The proof is provided in \cref{appendix: proof of curse of ambiguity}.
\Cref{thm:curse of ambiguity} suggests that ambiguous NeSy tasks may be unlearnable when the hypothesis space is very complex, such as nearest neighbor, whose Vapnik-Chervonenkis dimension is infinite~\citep{Karacali03NN}, or deep neural networks~\citep{bartlett2003vapnik} without any regularization terms.
This issue arises due to overfitting caused by the high memorization capacity of models~\citep{zhang2021understanding}.
Previous studies emphasize the importance of constraining the hypothesis space in NeSy tasks~\citep{yang24abl}.
For example, pre-training models or self-supervised learning methods~\citep{Sohn20Fixmatch} have been shown to promote clustering properties in neural networks, further enhancing generalization performance.

Consider a scenario where a pre-trained model satisfies a clustering property~\citep{huang2021ablsim}, meaning that instances representing the same concept are grouped together in feature space.
In the ambiguous task described in \cref{example:sum2}, if the model correctly processes a key sample such as \(\mathtt{SUM}(\IMGO,\IMGO)=\mathtt{0}\), it can reliably identify \(\mathtt{0}\).
This, in turn, simplifies subsequent tasks.
For example, once the model recognizes \(\mathtt{SUM}(\IMGO,\IMGI)=\mathtt{1}\), it can correctly identify \(\mathtt{1}\).
By iteratively applying this process, the model can learn to recognize all relevant concepts despite initial ambiguity.

The above process highlights the need to restrict the hypothesis space for the learning system. 
This hypothesis space ensures consistent mappings between concepts and labels, which can be formalized as follows:

\begin{definition}
  Let \(\F^*\) be a restricted hypothesis space which ensures that instances with the same label correspond to the same concept, and vice versa.
  Given the labeling function \(g\), formally, for any \(f \in \F^*\):
  \[
    \forall x_1, x_2 \in \X, \quad g(x_1) = g(x_2) \iff f(x_1) = f(x_2).
  \]
\end{definition}

\subsection{Derived Constraint Satisfaction Problem}
The restricted hypothesis space implicitly partitions the raw input space \(\X\) into \(L\) clusters.
Here we use \(\left<x\right>_i\) to denote the cluster \(\{x \mid x \in \X, f(x) = i\}\).
The learning process is to establish a mapping between the clusters \(\{\left<x\right>_1, \dots, \left<x\right>_{L}\}\) and \(\Z\) that minimize the NeSy risk.
This process inherently transforms the NeSy learning problem into a \emph{constraint satisfaction problem}~(CSP).
In this paper, we referred to it as derived CSP~(DCSP).

\begin{definition}
  The derived constraint satisfaction problem for a NeSy task \(\nesytask\) is defined as a triple \(\left<\mathsf{V}, \mathsf{D}, \mathsf{C}\right>\), where:
  \begin{itemize}[itemsep=0.2ex,topsep=0.1ex,leftmargin=1.2em]
    \item \(\mathsf{V} = \{V_1, \dots, V_{L}\}\) are the variables,
    \item \(\mathsf{D} = \{D_1 = \Z, \dots, D_{L} = \Z\}\) are the domains, and
    \item \(\mathsf{C} = \{C_1, \dots, C_N\}\) are the constraints.
  \end{itemize}
  Each \(V_i\) corresponds to a mapping from \(\left<x\right>_i\) to a concept label. For convenience, we slightly abuse notation by letting \(\mathsf{V}(\x)\) denote a mapping from an input sequence to the corresponding concept sequence determined by the mapping set \(\mathsf{V}\).
  Each \(C_j\) corresponds to a constraint \(\left(\x_j, y_j\right)\), e.g., \(\mathsf{V}(\x_j) \land \KB \models y_j\)\,.
  Solving the DCSP is to find a consistent assignment \(I\) that satisfies all constraints.
\end{definition}

A DCSP solution \(I\) corresponds to an assignment of values to variables, expressed as \(I = \{(V_1, v_1), \dots, (V_{L}, v_{L})\}\), where each \(v_i\) is the value assigned to the variable \(V_i\).
For simplicity, we denote the solution as \(I = (v_1, \dots, v_{L})\) by omitting the variables.
Here we only discuss the case when the DCSP has solution;
Otherwise, the learning model will inevitably conflict with the background \KB.

\begin{restatable}{assumption}{assumzero}{\normalfont (No Conflict)\;}\label{assumption: non conflict}
  The DCSP has solution.
\end{restatable}

\subsection{Conditions of Learnability}\label{sec: conditions of learnability}
In general, the solution to a DCSP may not be unique, meaning multiple distinct solutions can exist.
We represent it as a solution space \(\mathcal{S} = \{I_1, \dots, I_k\}\).
To handle the relationships between these solutions, we define an operation, \(\textrm{Union}()\), which captures the common assignments among the solutions. When the given input set consists of a single element, this operation simply returns that element.

\begin{definition}
  The \emph{DCSP solution disagreement} \(d\) quantifies the inconsistency among all solutions:
  \begin{equation*}
    d = L - |\textrm{Union}(\mathcal{S})|\,.
  \end{equation*}
\end{definition}

The disagreement \(d\) measures how many variables have different values across the solutions in \(\S\).
If \(d = 0\), which means \(|\mathcal{S}| = 1\), there is only one solution.
This means the optimal hypothesis can be determined by minimizing the NeSy risk, formally we have:

\begin{restatable}{lemma}{lemmaCr}\label{thm:consistent risk}
  For a NeSy task \nesytask{}, if the DCSP solution disagreement \(d = 0\), then the NeSy risk is equivalent to the concept risk. Formally, for any \(f \in \F\):
  \begin{equation*}
    \Nesyrisk(f) \to 0 \iff \truerisk(f) \to 0.
  \end{equation*}
\end{restatable}

\begin{proof}[Proof Sketch]
  The direction from the right-hand side to the left-hand side is straightforward; here, we focus on proving the reverse direction.
  We demonstrate this by showing that \emph{if the concept risk is non-zero, then the NeSy risk cannot be zero}~(contraposition).
  If the concept risk is non-zero, there must be at least one misclassified instance where \(f\) assigns an incorrect label. Given that the DCSP solution is unique and there is no disagreement (i.e., \(d = 0\)), any such misclassification directly results in a non-zero NeSy risk. Therefore, if the NeSy risk is zero, it follows that the concept risk must also be zero.
\end{proof}

The detailed proof is provided in \cref{appendix: proof of consistent risk}.
To further investigate the learnability of NeSy tasks, we introduce the following mild assumptions.

\begin{restatable}{assumption}{assumone}{\normalfont (Finite Cardinality)\;}\label{assumption: finite size}
  The set of possible concept sequences, \(\B = \bigcup_{y \in \Y} A(y)\), has finite cardinality.
\end{restatable}
\begin{restatable}{assumption}{assumtwo}{\normalfont (Non-vanishing Probability)\;}\label{assumption: non-zero prob}
  The sampling process is controlled by a distribution \(P(Z)\), and for any concept sequence \(\z \in \B\), the probability of being sampled is at least a small positive constant \(\kappa\).
\end{restatable}

Under these assumptions, we formally present the main result of this paper as follows.

\begin{restatable}{theorem}{thmLearnable}\label{thm: learnable}
  For a neuro-symbolic task \(\mathcal{T}\) with a restricted hypothesis space \(\F^*\), the learnability is determined by the conditions:
  \begin{itemize}[itemsep=0.2ex,topsep=0.1ex,leftmargin=1.2em]
    \item
      If the derived constraint satisfaction problem has a unique solution, the task is \emph{learnable}.
      Specifically, the concept error is bounded by \(\epsilon\), given that the sample size \(N\) satisfies the following condition:
      \begin{equation*}
        N > \frac{1}{\kappa} \cdot \log\left({|\B|}/{\epsilon}\right)\,.
      \end{equation*}
    \item
      Otherwise, the task is \emph{unlearnable}.
  \end{itemize}
\end{restatable}

The proof is in \cref{appendix: proof of learnable}\,.
\Cref{thm: learnable} establish that a NeSy task \(\mathcal{T}\) is \emph{learnable} if and only if the DCSP solution is unique, i.e., disagreement \(d = 0\)\,.
Conversely, if the DCSP has multiple solutions (i.e., \(d \geq 1\)), the task is \emph{unlearnable}, implying that concept error remains \emph{unavoidable} regardless of additional training data.


Building upon the concept of \ambiguityDegree{}, we derive a more general theorem offering deeper insights into learning errors in a restricted hypothesis space \(\F^*\).
As the sample size approaches infinity, the hypotheses learned via ERM asymptotically converge to:
\[
  \F^*_{\text{ERM}} = \arg\min_{f \in \F^*} \Nesyrisk(f).
\]
The average error of the ERM result, denoted by \(\mathcal{E}^*\), is the expected concept risk of an arbitrarily selected hypothesis:
\[
  \mathcal{E}^* = \operatorname{\Exp}_{f \in \F^*_{\text{ERM}}} \left[\truerisk(f)\right].
\]

\begin{restatable}{theorem}{thmAEB}\label{thm: average error bound}
  The average error \(\mathcal{E}^*\) is bounded by:
  \begin{equation*}
    \mathcal{E}^* \leq \frac{d}{L}\,.
  \end{equation*}
\end{restatable}
\begin{proof}
  Recall that the DCSP solution disagreement \(d\) is given by \(d = L - \textrm{Union}(\mathcal{S})\), where \(\textrm{Union}(\mathcal{S})\) represents the common assignments among the solutions.  
  Since the restricted hypothesis space ensures that instances with the same assigned label correspond to the same concept and vice versa.  
  In the worst case, errors occur in at \(d\) classes, so the maximum true risk is \( \max_{f \in \F^*_{\text{ERM}}} \truerisk(f) = {d}/{L} \).  
  Thus, the average error is bounded by:
  \(
    \mathcal{E}^* \leq \max_{f \in \F^*_{\text{ERM}}} \truerisk(f) = {d}/{L}.
  \)
\end{proof}
\Cref{thm: average error bound} provides an asymptotic error analysis for NeSy tasks, indicating that as the \ambiguityDegree{} \(d\) increases, the upper bound of the concept error also increases.
Revealing that the disagreement \(d\) is crucial to understanding the learnability of NeSy tasks.

\subsection{Examples}

Here we present some examples for better understanding the learnability conditions of a NeSy task.
To demonstrate the distinction between \emph{learnable} and \emph{unlearnable} tasks, we utilize digital images as input data to give a few examples.
The data is modeled as \(\x = (x_1, x_2) \in \X^2\), where \(\X\) represents the space of digit images, e.g., \(\{\IMGO, \IMGI, \dots\}\).
The intermediate concept space \(\Z\) and the label space \(\Y\) depend on the specific knowledge base.

Table~\ref{tab:learnability_examples} provides a summary of the examples.
Learnable tasks are straightforward to verify, as prior research has demonstrated the effectiveness of NeSy methods for these tasks~\citep{deepproblog_aij, Li2023Scallop, a3bl}. For the unlearnable cases, we present two distinct solutions to the DCSP for each example.

For the XOR task~(\(  d/L=1\)), if the concepts \(\mathtt{0}\) and \(\mathtt{1}\) are interchanged, e.g., \((\IMGO \mapsto \mathtt{0}, \IMGI \mapsto \mathtt{1})\) and \((\IMGO \mapsto \mathtt{1}, \IMGI \mapsto \mathtt{0})\), the NeSy risk can be minimized.

For the modular addition task (\(k = \mathtt{9}, d/L=0.2\)), if the mappings of \(\mathtt{0}\) and \(\mathtt{9}\) are swapped, e.g., \((\IMGO \mapsto \mathtt{0}, \IMGNine \mapsto \mathtt{9})\) and \((\IMGO \mapsto \mathtt{9}, \IMGNine \mapsto \mathtt{0})\), the NeSy risk can be minimized.

\begin{table}[!htb]
  \centering
  \caption{Examples of learnable and unlearnable tasks. In specific, the modular addition task requires \(2\leq k\leq 10\).}
  \label{tab:learnability_examples}
  \centering
  \resizebox{.5\linewidth}{!}{
  \begin{tabular}{lll}
    \toprule
    {Category} & {Task} & {Knowledge Base} \\ \midrule
    \multirow{2}{*}{\emph{Learnable}}
    & Addition & \( y = z_1 + z_2 \) \\
    & Multiplication & \( y = z_1 \times z_2 \) \\ \midrule
    \multirow{2}{*}{\emph{Unlearnable}}
    & Exclusive OR & \( y = z_1 \oplus z_2 \) \\
    & Modular Addition & \( y = (z_1 + z_2)\;\mathtt{mod}\; k\) \\
    \bottomrule
  \end{tabular}
  }
\end{table}

\subsection{Ensemble Unlearnable Tasks}\label{sec: ensemble tasks}

Some NeSy tasks are inherently unlearnable because they admit multiple solutions to their DCSPs, resulting in ambiguity. This ambiguity cannot be resolved by increasing data or improving the learning algorithm, as it stems from intrinsic task properties.
Interestingly, such unlearnable tasks may become learnable when combined in an ensemble framework under a multi-task learning paradigm. The key insight is that tasks can mutually constrain each other, reducing ambiguity.

Consider two unlearnable tasks \(\nesytask_1\) and \(\nesytask_2\) with corresponding solution spaces \(\mathcal{S}_1\) and \(\mathcal{S}_2\). Each task is individually unlearnable, i.e., \(|\mathcal{S}_1| \geq 2\) and \(|\mathcal{S}_2| \geq 2\). In a multi-task learning setting, the model is required to satisfy both tasks simultaneously, thereby constraining the solution to the intersection \(\mathcal{S}_1 \cap \mathcal{S}_2\). This intersection reduces ambiguity and can potentially yield a unique solution. Therefore, by leveraging the unique constraints induced by the intersection of solution spaces, combining unlearnable tasks into an ensemble framework may enable learning.

From the perspective of DCSP, we can formally state the conditions when tasks can ensemble to become learnable:
\begin{corollary}\label{thm: ensemble}
    NeSy tasks become learnable in an ensemble framework if combining their DCSPs results in a unique solution.
\end{corollary}

\section{Empirical Study}\label{sec:experiments}
To empirically validate the theoretical results, we conducted a series of experiments.
Due to space limitations, some experimental results are presented in the appendix.

\paragraph{Setup}
\begin{figure*}[!tb]
  \centering
  \begin{subfigure}[b]{\linewidth}
    \includegraphics[width=\linewidth]{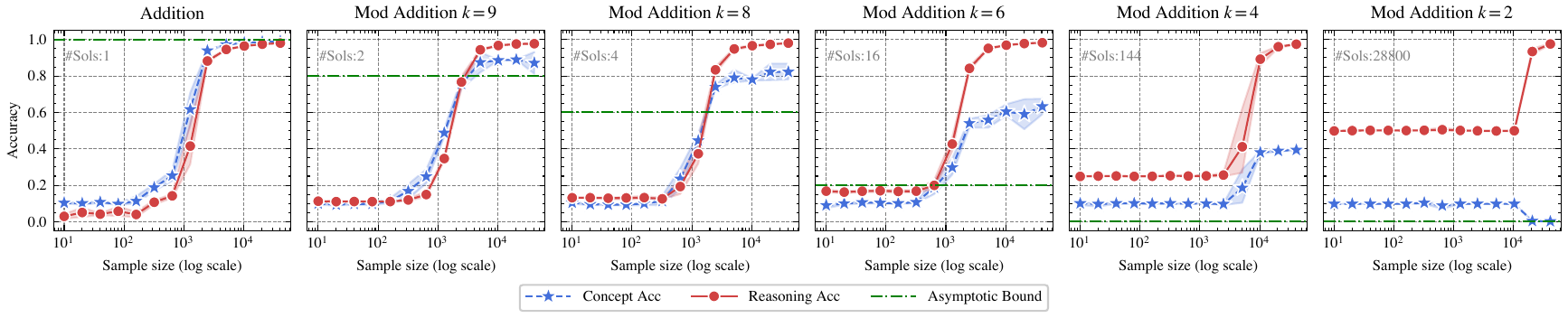}
  \end{subfigure}
  \begin{subfigure}[b]{\linewidth}
    \includegraphics[width=\linewidth]{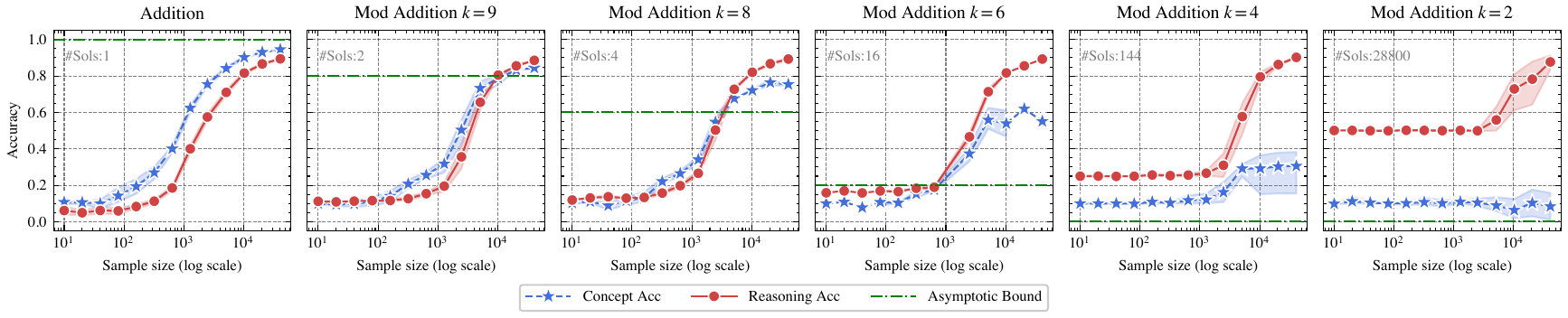}
  \end{subfigure}
  \caption{\emph{Accuracies} versus \emph{sample size} for different NeSy tasks~(top MNIST and bottom KMNIST).
    The shadowed area denotes the standard error.
    The number of the DCSP solutions~({\small\color{gray}\#Sols}) is shown at the top left of each plot.
  The asymptotic bound~({\color{sgreen}green} line) from \cref{thm: average error bound} indicates that concept accuracy should exceed this bound as the sample size grows.}
  \label{fig:sample_complexity}
\end{figure*}
\citet{manhaeve_deepproblog_2018} proposed the Addition task by incorporating the handwritten MNIST~\citep{deng2012mnist} and predefined addition rules.
We extend the setup by including KMNIST~\citep{kmnist}, CIFAR10~\cite{krizhevsky2009cifar}, and SVHN~\citep{SVHN}, mapping class indices to digits and enriching the background knowledge as depicted in \cref{tab:learnability_examples}.
The learning model for MNIST and KMNIST is LeNet~\citep{lecun1995convolutional}, while ResNet50~\citep{he2016deep} is used for CIFAR10 and SVHN.
The results were obtained using an Intel Xeon Platinum 8538 CPU and an NVIDIA A100-PCIE-40GB GPU on an Ubuntu 20.04 platform.
All experiments were conducted five times with different random seeds.
More details can be seen in \cref{appendix: detailed experiments}.

\paragraph{Method}
To effectively optimize the NeSy risk~\cref{eq: nesy risk}, we adopt the following surrogate~(cf. proof in \cref{appendix: proof of surrogate}):
\begin{equation}\label{eq: a3bl}
  -\Exp_{(\x,y)} \log \left(\sum_{\bar{\z} \in \mathcal{N}(y)}\Pr\left[ y, \bar{\z}\mid \x;f,\KB\right]\right)\,,
\end{equation}
which is flexible, where \(\mathcal{N}(y) \subseteq A(y)\) represents several valid candidates for the final answer \(y\).
By restricting the size of \(\mathcal{N}(y)\) from the entire set \(A(y)\) to the most likely candidate \(\bar\z\), we achieve a balance between PNL and ABL, and we set the size of \(\mathcal{N}(y)\) is \(\min\left(16, |A(y)|\right)\).
The implementation is based on the code of \citet{a3bl}.
For brevity, detailed experiments on PNL and ABL are provided in \cref{appendix: detailed experiments}.

\subsection{Empirical Analysis on Learnability}
\begin{figure}[!htb]
    \centering
    \includegraphics[width=.55\linewidth]{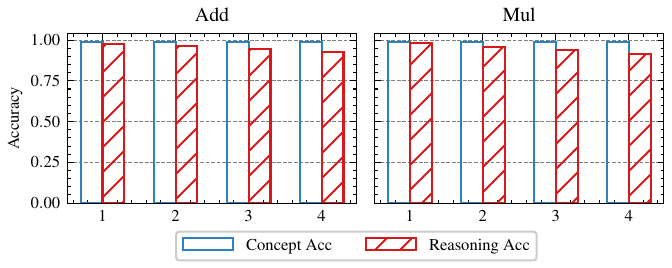}
    \caption{Accuracies on the learnable tasks. }
    \label{fig:learnability}
  \end{figure}
We empirically evaluate the learnability of NeSy tasks based on \cref{thm: learnable}, focusing on two key aspects: (i) validating that minimizing the NeSy risk consistently minimizes the concept risk for learnable tasks, and (ii) examining how DCSP solution disagreement affects learnability.

\emph{(i) Validation of Learnable Tasks.}
We first validate the learnability conditions~(cf. \cref{thm: learnable}) by examining addition and multiplication tasks (cf. \cref{tab:learnability_examples}). 
By solving the DCSP shows that both tasks are learnable, and their learnability remains unaffected by increases in digit size (e.g., from \(\mathtt{PROD}(\IMGI,\IMGII)=\mathtt{2}\) to \(\mathtt{PROD}(\IMGI\IMGO,\IMGII\IMGO)=\mathtt{200}\)). The raw dataset in \cref{fig:learnability} is MNIST, and additional results for other datasets are in the appendix.
We further substantiate learnability by examining tasks with varying digit sizes, ranging from one to four digits. 
As depicted in \cref{fig:learnability}, the results confirm that:
(a) Optimization of the surrogate risk \cref{eq: a3bl} effectively minimizes the NeSy risk.
(b) For learnable tasks, a good minimizer of the NeSy risk also serves as a reliable minimizer of the concept risk.

\begin{figure*}[!tb]
    \centering
    \begin{subfigure}[b]{\linewidth}
       \label{fig: ensemble of unlearnable tasks: unaffectedlearnable case}
      \begin{subfigure}[b]{.495\linewidth}
        \centering
        \includegraphics[width=\linewidth]{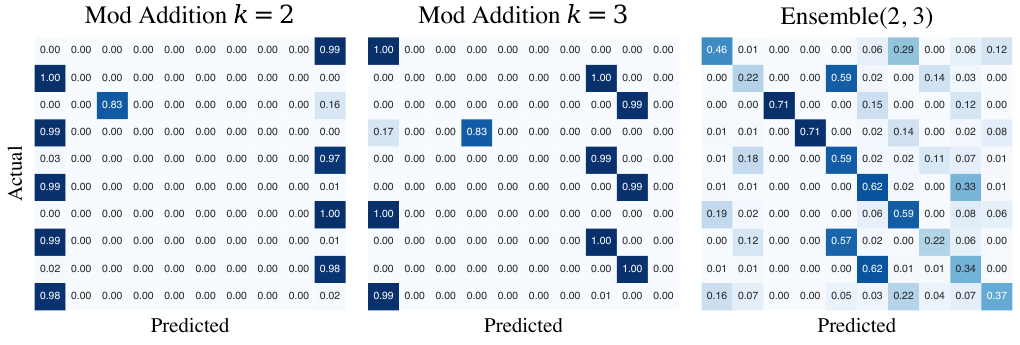}
      \end{subfigure} \hfill
      \begin{subfigure}[b]{.495\linewidth}
        \centering
        \includegraphics[width=\linewidth]{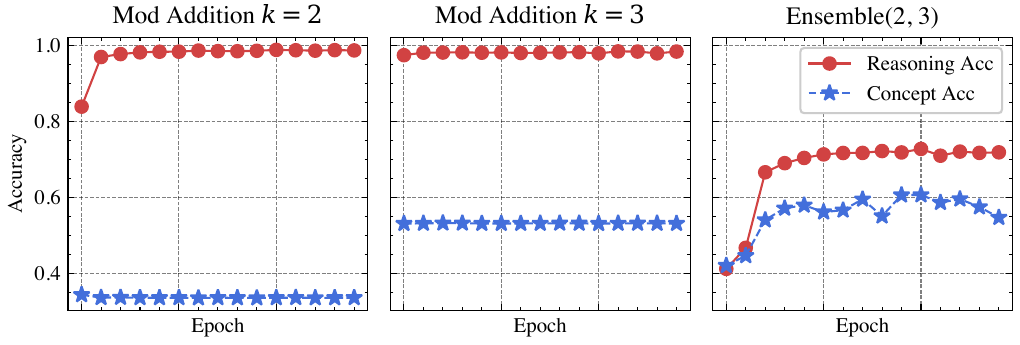}
      \end{subfigure}
    \end{subfigure} \vfill
  
    \begin{subfigure}[b]{\linewidth}
       \label{fig: ensemble of unlearnable tasks: learnable case}
      \begin{subfigure}[b]{.495\linewidth}
        \centering
        \includegraphics[width=\linewidth]{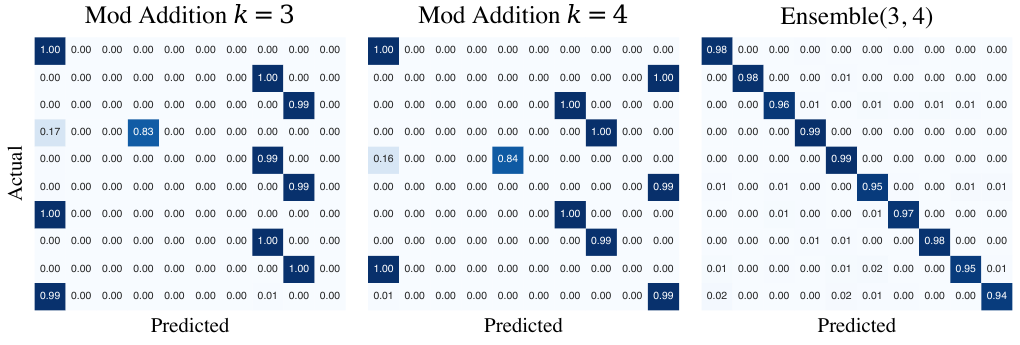}
      \end{subfigure} \hfill
      \begin{subfigure}[b]{.495\linewidth}
        \centering
        \includegraphics[width=\linewidth]{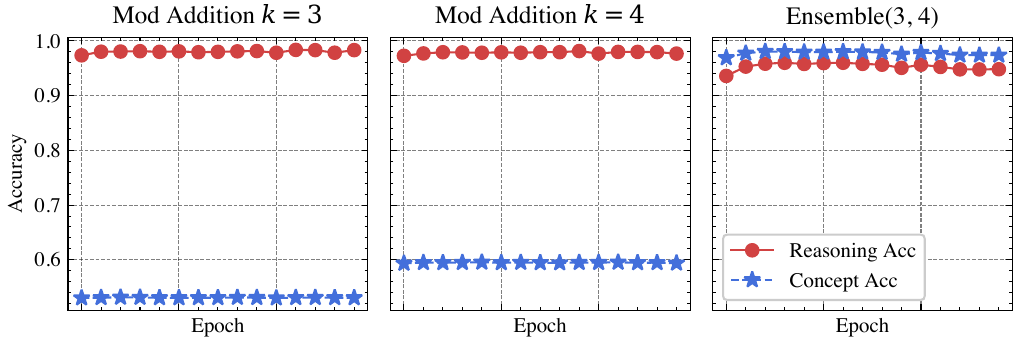}
      \end{subfigure}
    \end{subfigure}
    \caption{\emph{Ensemble} of unlearnable NeSy tasks. 
    The left shows confusion matrices and the right displays accuracy curves. 
    (a) The top row illustrates an \emph{unlearnable} case, where combining the tasks still results in multiple DCSP solutions. 
    (b) The bottom row illustrates a \emph{learnable} case, where combining the tasks reduces the DCSP solutions to a single one.}

    \label{fig: ensemble of unlearnable tasks}
  \end{figure*}
\emph{(ii) Impact of DCSP Solution Disagreement.}
We further investigate how disagreement in DCSP solutions impacts learnability. 
According to \cref{thm: average error bound}, the asymptotic error is bounded by the ratio of DCSP solution disagreement \(d\) to the size of the concept space \(L\).
Experiments involving addition and modular addition tasks with varying modular bases \(k\) reveal that altering the knowledge base changes the DCSP solution space, directly influencing learnability.
For clarity, we plot the asymptotic accuracy bound for each task, i.e., \(1 - d/L\), showing that higher disagreement results in a lower bound line~(green).
As shown in \cref{fig:sample_complexity}:
(a) Tasks with a unique DCSP solution are learnable;
(b) Tasks with high DCSP disagreement struggle to achieve low concept risk, even as the sample size increases.

In summary, our empirical analysis confirms the theoretical learnability conditions by demonstrating that minimizing the NeSy risk reliably minimizes the concept risk for learnable tasks; 
Furthermore, tasks with lower disagreement exhibit better learnability, while those with high disagreement suffer from ambiguity due to multiple conflicting solutions.

\subsection{Ensemble of Unlearnable NeSy Tasks}
With the DCSP framework, we find that
certain NeSy tasks are inherently unlearnable because their DCSPs admit multiple solutions, resulting in inherent ambiguity. However, when combined in an ensemble framework within a multi-task learning setting, such tasks may become learnable by enforcing mutual consistency, as shown in \cref{thm: ensemble}. 
We evaluate \cref{thm: ensemble} using mod addition tasks with mod bases \(k_1\) and \(k_2\) under two specific configurations: an unlearnable ensemble (\(k_1 = \mathtt{2}\), \(k_2 = \mathtt{3}\)) and a learnable ensemble (\(k_1 = \mathtt{3}\), \(k_2 = \mathtt{4}\)). For \(k = \mathtt{2, 3, 4}\), the degree of DCSP solution disagreement \(d\) is \(10\). The experiments in \cref{fig: ensemble of unlearnable tasks} are based on the raw MNIST dataset.
Additional details and experiments are provided in \cref{app:details of ensemble experiments}.

In the top of \cref{fig: ensemble of unlearnable tasks}, the unlearnable case (\(k_1 = \mathtt{2}\), \(k_2 = \mathtt{3}\)) shows that while the ensemble narrows the solution space, reducing the disagreement \(d\) to \(8\), it does not converge to a unique solution, and the task remains unlearnable.

In the bottom of \cref{fig: ensemble of unlearnable tasks}, the learnable case (\(k_1 = \mathtt{3}\), \(k_2 = \mathtt{4}\)) illustrates that both tasks initially admit multiple DCSP solutions, causing reasoning accuracy to exceed concept accuracy, as shown in \cref{fig: ensemble of unlearnable tasks}. Following the ensemble, the intersection of solution spaces yields a unique solution, with the disagreement \(d\) reduced to \(0\), rendering the ensemble task learnable.

This experimental result supports \cref{thm: ensemble}, demonstrating that forming ensembles of different NeSy tasks can enhance learnability by mutually constraining DCSP solution spaces. 
This finding suggests that we can collect tremendous NeSy tasks and jointly learn them in an ensemble manner, which could potentially introduce a ``scaling law’’~\citep{kaplan2020scalinglawsneurallanguage} in the NeSy domain.
\section{Limitation}\label{sec: limitation}

This paper focuses exclusively on hybrid neuro-symbolic systems, e.g., probabilistic neuro-symbolic and abductive learning methods.
Thus the findings may not directly extend to other types of neuro-symbolic methods. 
The analysis of this study relies on a restricted hypothesis space, which is inherently satisfied by models such as neural networks equipped with manifold regularization~\citep{belkin2006manifold} or self-supervised pretraining~\citep{liu2021self}. 
However, extending the framework to encompass more general hypothesis spaces without requiring this specific property remains an open challenge. 

Future work may involve a deeper investigation into extending the learnability framework to encompass a broader range of NeSy systems.
Additionally, exploring the learnability of the semi-supervised case of NeSy tasks, where some training examples are supervised for intermediate concepts, could be an interesting direction.
Developing practical strategies for constructing effective task ensembles aslo represents a promising avenue for improving learnability in diverse and complex scenarios.

\section{Conclusion}\label{sec: conclusion}

We disclose that a neuro-symbolic (NeSy) task is learnable if and only if the derived constraint satisfaction problem (DCSP) has a unique solution.
Using the DCSP framework, we can determine whether a NeSy task is learnable and derive the task-specific asymptotic bound for the concept error.
Moreover, we find that constructing ensembles of previously unlearnable tasks reduces the degree of ambiguity, thereby enhancing overall task learnability. 
Experimental results validate our theoretical findings.

\section*{Impact Statement}
This paper presents a theoretical analysis of the \emph{learnability} of neuro-symbolic learning, focusing on hybrid systems such as probabilistic neuro-symbolic learning and abductive learning methods.
This work's impact lies in providing theoretical insights into these systems and elucidating the underlying mechanisms of the recently observed "reasoning shortcut" phenomenon.
We do not anticipate that this work will introduce any negative ethical or social impacts.
\bibliography{strings,ref}
\bibliographystyle{icml2025}

\clearpage
\newpage

\appendix
\onecolumn

\part*{\centering \Large \textbf{Appendix}} \label{appendix}
The appendix is structured as follows.
\begin{itemize}
    \item \Cref{appendix: proof} contains proofs omitted in the main paper, because of the space limit.
    \item \Cref{appendix: detailed experiments} contains more details and additional experiments.
\end{itemize}
\section{Proofs}\label{appendix: proof}
In this section, we give the proofs that are omitted in the main text.
For the convenience of the reader, we re-state the assumptions, lemmas, propositions, and theorems in the appendix again.

\subsection{Proof of \Cref{thm: surrogate}}\label{appendix: proof of surrogate}
\thmsurrogate*
\begin{proof}
  First, we recall the risks as follows:
  \begin{align*}
    \PNLrisk (f) &=  -\Exp_{(\x,y)} \log \sum_{\z} \Ind(\z \land \KB \models y) \cdot \Pr\left[\z \mid \x; f, \KB\right]\,, \\
    \ABLrisk(f) &= -\Exp_{(\x,y)} \log \left(\Pr\left[ y, \bar{\z}\mid \x;f,\KB\right]\right)\,.\\
  \end{align*}

  To unify the proof, we introduce \cref{eq: a3bl} as a surrogate form:
  \begin{equation*}
    \AAABLrisk(f) = -\Exp_{(\x,y)} \log \left(\sum_{\bar{\z} \in \mathcal{N}(y)}\Pr\left[ y, \bar{\z}\mid \x;f,\KB\right]\right)\,,
  \end{equation*}
  where the set \(\mathcal{N}(y)\) denotes the candidate set satisfying:
  \[ \forall \z \in \mathcal{N},\, \z \land \KB \models y\,.\]

  In most practical scenarios, due to time or computational constraints, this set may not include all possible candidates.
  Furthermore, we reformulate the PNL risk using a union-based representation:
  \begin{equation*}
    \begin{aligned}
      \PNLrisk (f) &=  -\Exp_{(\x,y)} \log \sum_{\z \in A(y)} \Ind(\z \land \KB \models y) \cdot \Pr\left[\z \mid \x; f, \KB\right] \\
      &=  -\Exp_{(\x,y)} \log \left(\sum_{\z \in A(y)} \Pr\left[y, \z \mid \x; f, \KB\right]\right)\,.
    \end{aligned}
  \end{equation*}

  Consequently, \AAABLrisk{} emerges as the most flexible surrogate form. By adjusting the size of the candidate set, we can interpolate between the ABL risk and the PNL risk. Thus, it suffices to prove that for any candidate set \(\mathcal{N}(y)\), \AAABLrisk{} achieves the desired objective.

  Since the following properties hold:
  \begin{itemize}
    \item For any \(f\in \F\), the risk \(\AAABLrisk(f) \geq 0\).
    \item For the labeling function \(g\), the risk \(\AAABLrisk(g) = 0\).
  \end{itemize}
  Therefore, the minimum achievable value of this risk is strictly \(0\). For any \(f^* \in \arg\min_{f\in\F} \AAABLrisk(f)\), we have \(\AAABLrisk(f^*) = 0\), which implies that, for a fixed candidate set \(\mathcal{N}(y)\) and any \((\x,y)\):
  \begin{equation*}
    \begin{aligned}
      &\sum_{\bar{\z} \in \mathcal{N}(y)}\Pr\left[ y, \bar{\z}\mid \x;f^*,\KB\right] \\
      = &\sum_{\bar{\z} \in \mathcal{N}(y)}\Ind(\bar{\z} \land \KB \models y) \cdot p_{\theta^*} (\bar{\z} \mid \x) = 1\,.
    \end{aligned}
  \end{equation*}

  Consequently, any \(\bar{\z}\) predicted by the learning model \(f^*\) with a probability greater than zero will satisfy the knowledge base, i.e., \(\bar{\z} \land \KB \models y\). This ensures that the NeSy risk \(\Nesyrisk(f) = -\operatorname{\Exp}\limits_{(\x,y)}\left[ \Ind\left(f^*(\x) \land \KB \not\models y\right)\right]\) is also zero.
  Since the NeSy risk should also be greater or equal to zero, which means the hypothesis \(f^*\) is a minimizer of NeSy risk.
  Hence, the proof is complete.
\end{proof}

\subsection{Proof of \Cref{thm:curse of ambiguity}}\label{appendix: proof of curse of ambiguity}
\propCoA*

\begin{proof}
  By the definitions of \(\Nesyrisk\) and \(\truerisk\), we have:
  \[\Nesyrisk(f) = \operatorname{\Exp}\limits_{(\x,y)}\left[ \Ind\left(f(\x) \land \KB \not\models y\right)\right],\]
  and, thus,
  \[\truerisk(f) = \Exp_{x,z}\left[\Ind(f(x) \neq z)\right].\]

  Since \(\mathcal{T}\) is ambiguous, i.e., there exists a \(y \in \Y\) such that \(|A(y)| \geq 2\), we assume, without loss of generality, a sample pair \((\x_0, y_0) \in (\X^m, \Y)\) such that \( \{\z_1, \z_2\} \subseteq A(y_0)\).

  Given that the hypothesis space \(\F\) is sufficiently complex to shatter the task, we assume the existence of two hypotheses \(f_1\) and \(f_2\) that yield identical correct predictions for all inputs except \(\x_0\):
  \[
    \begin{cases}
      f_1(\x_0) = \z_1, \quad f_2(\x_0) = \z_2 \quad & \text{if } \x = \x_0, \\
      f_1(\x) = f_2(\x)                              & \text{otherwise.}
    \end{cases}
  \]

  By definition, as \(f_1\) and \(f_2\) yield identical predictions except at \(\x_0\), we have \(\Nesyrisk(f_1) = \Nesyrisk(f_2)\).
  However, since \(\z_1 \neq \z_2\), there exists at least one index \(k \in [m]\) such that \((\z_1)_k \neq (\z_2)_k\).
  Thus, by the definition of \(\truerisk\), we observe that \(\truerisk(f_1) \neq \truerisk(f_2)\), as at the sample \((\x_0)_k\), they produce two distinct recognition results.

  In this scenario, even if \(f_1\) represents the underlying ground truth mapping function, it is indistinguishable from \(f_2\), as both achieve zero risk under the optimized objective \(\Nesyrisk\).
  This concludes that \( \left(\Nesyrisk \to 0\right) \not\Rightarrow \left(\truerisk \to 0\right)\).
\end{proof}

\subsection{Proof of \Cref{thm:consistent risk}}\label{appendix: proof of consistent risk}
\lemmaCr*
\begin{proof}
  We first prove the direction:
  \[
    \left(\Nesyrisk \to 0\right) \Leftarrow \left(\truerisk \to 0\right)\,.
  \]
  This is evident because, as \(\truerisk\) approaches zero, \(f\) must correctly classify all input-output pairs, which consequently drives \(\Nesyrisk\) to zero as well.

  Next, we prove the direction:
  \(
    \left(\Nesyrisk \to 0\right) \Rightarrow \left(\truerisk \to 0\right).
  \)
  Equivalently, we prove the contrapositive:
  \[
    \left(\truerisk \not\to 0\right) \Rightarrow \left(\Nesyrisk \not\to 0\right)\,.
  \]

  Suppose \(\truerisk \not\to 0\). Then, there exist integers \(i, j \in [L]\) with \(i \neq j\) such that \(f\) misclassifies elements of the set
  \[
    \left<x\right>_i = \{x \mid x \in \X, g(x) = i\}
  \]
  as belonging to class \(j\).

  Recall that the DCSP of \(\mathcal{T}\) has a unique solution, which ensures that the correct labels are unambiguous. As the training set size grows, there must exist a sample \((\x, y) \in (\X^m, \Y)\) such that \(\text{set}(\x) \cap \left<x\right>_i \neq \emptyset\) and \(f(\x) \not\in A(y)\). This implies that \(\Nesyrisk \not\to 0\).

  Thus, by proving both directions, we complete the proof.
\end{proof}

\subsection{Proof of \Cref{thm: learnable}}\label{appendix: proof of learnable}
First, we recall that the learnability analysis depends on two assumptions.
\assumone*
\assumtwo*

Based on the assumptions, we first prove the below lemma, which states the sample complexity under the learnable case, when the hypothesis space is restricted hypotheses space.
\begin{lemma}\label{thm: sample complexity of learnable case}
  Consider a NeSy task \(\nesytask{}\) with above assumptions and \(d=0\).
  By applying empirical risk minimization, the hypothesis \(\hat{f} = \arg\min_{f \in \F^*} \empNesyrisk(f)\) ensures that \(\truerisk(\hat{f}) \leq \epsilon\,\) for any \(\epsilon > 0\), provided that the training size \(N\) satisfies the inequality:
  \begin{equation*}
    N > \frac{1}{\kappa} \cdot \log\left(\frac{|\B|}{\epsilon}\right)\,.
  \end{equation*}
\end{lemma}

\begin{proof}
  Recall that empirical risk minimization, based on the neuro-symbolic risk \(\empNesyrisk\), corresponds to solving a derived constraint satisfaction problem over the restricted hypothesis space \(\F^*\).
  By \cref{thm:consistent risk}, if the training set includes all possible concept sequences, the minimum value of \(\Nesyrisk\) becomes zero.
  This ensures that \(\truerisk\) also attains a value of zero. Therefore, it is crucial to analyze the sampling process of the training data.

  Let \(Q\) denote the event that \emph{not all concept sequences are sampled in the training data}.
  Therefore, we conclude that the true risk is bounded by the probability of event \(Q\):
  \[\truerisk(\hat{f}) \leq \Pr[Q].\]

  To bound \(\truerisk\), it suffices to bound \(\Pr[Q]\). For any individual concept sequence \(\z_i\), the probability that it is not sampled after \(N\) draws is given by:
  \[\left(1 - p_i\right)^N \leq \left(1 - \kappa\right)^N\,.\]

  Applying the union-bound inequality, we derive:
  \[\Pr[Q] \leq |\B| \left(1 - \kappa\right)^N\,.\]

  Since \((1 - x) \leq \exp(-x)\) holds for \(x \geq 0\), we can further bound \(\truerisk(\hat{f})\) as follows:
  \[\truerisk(\hat{f}) \leq \Pr[Q] \leq |\B| \exp(-{N}\cdot \kappa)\,.\]

  Given that \(\truerisk(\hat{f}) \leq \epsilon\), it follows that:
  \[N \geq \frac{1}{\kappa}\cdot \log\left(\frac{|\B|}{\epsilon}\right)\,.\]

  This completes the proof of the proposition.
\end{proof}

\thmLearnable*
\begin{proof}
  The proof is divided into two parts:
  \begin{enumerate}
    \item If the disagreement \(d\) equals zero, the task is \emph{learnable}, and the sample complexity is \(\mathcal{O}(\frac{1}{\kappa}\cdot \log(|\mathcal{B}|/\epsilon))\).
    \item If the disagreement \(d\) is greater than zero, the DCSP solution space contains at least two distinct solutions, making the task \emph{unlearnable}.
  \end{enumerate}

  The first part follows directly from \cref{thm: sample complexity of learnable case}.
  Hence, we focus on proving the second part by contradiction.

  If the DCSP has multiple solutions, there exists \((\x,y) \in (\X^m, \Y)\) such that two distinct concept sequences \(\z_1\) and \(\z_2\) are valid, i.e., \(\z_1 \land \KB \models y\) and \(\z_2 \land \KB \models y\).

  Since both \(\hat{f}_1(\x) = \z_1\) and \(\hat{f}_2(\x) = \z_2\) are valid solutions for \((\x,y)\), and \(\z_1\) and \(\z_2\) are distinct, it follows that their true risks cannot be simultaneously zero. Thus, at least one of them must have a \(\truerisk\) greater than zero.
  Without loss of generality, assume that \(\truerisk(\hat{f}_1) = \epsilon_0 > 0\).

  Since both \(\hat{f}_1\) and \(\hat{f}_2\) achieve the minimal NeSy risk (which is zero), it is impossible to distinguish between them using learning techniques or by adding more data.
  Consequently, there is no integer \(N_\epsilon\) such that for any \(0 < \epsilon < \epsilon_0\), \(\truerisk(\hat{f}) < \epsilon\) holds when \(N \geq N_\epsilon\). This implies that the task is unlearnable.

  Combining both parts completes the proof.
\end{proof}

\subsection{Proof of \Cref{thm: average error bound}} \label{appendix: proof of average_error_bound}
\thmAEB*

\begin{proof}
  Recall that the DCSP solution disagreement \(d\) is given by \(d = L - \textrm{Union}(\mathcal{S})\), where \(\textrm{Union}(\mathcal{S})\) represents the common assignments among the solutions.
  Since the restricted hypothesis space ensures that instances with the same assigned label correspond to the same concept and vice versa.
  In the worst case, errors occur in at \(d\) classes, so the maximum true risk is \( \max_{f \in \F^*_{\text{ERM}}} \truerisk(f) = {d}/{L} \).
  Thus, the average error is bounded by:
  \(
    \mathcal{E}^* \leq \max_{f \in \F^*_{\text{ERM}}} \truerisk(f) = {d}/{L}.
  \)
\end{proof}

\clearpage
\section{Experiments}\label{appendix: detailed experiments}
We first introduce the experimental details, including data preparation, model setup, optimizer configurations, hyperparameters, and implementation details.
After that, we present experiments omitted from the main context due to space constraints.

\subsection{Experiment Details}
The raw datasets are based on MNIST~\citep{deng2012mnist}, KMNIST~\Citep{kmnist}, CIFAR-10~\Citep{krizhevsky2009cifar}, and SVHN~\citep{SVHN}.
For MNIST-style datasets, the learning model is based on LeNet~\citep{lecun1995convolutional}; other datasets use ResNet~\citep{he2016deep}.
The results were obtained using an Intel Xeon Platinum 8538 CPU and an NVIDIA A100-PCIE-40GB GPU on an Ubuntu 20.04 platform.

\subsubsection{Preparing data and model}
The construction of datasets is heavily based on algorithmic operations; thus, we rely on digit indices mapping from class indices to digit indices.
After that, different knowledge bases require different rules. Here, we base our approach on ABLKit~\citep{ABLkit2024}\footnote{\tiny\url{https://github.com/AbductiveLearning/ABLkit}} and the code of \citet{a3bl}\footnote{\tiny\url{https://github.com/Hao-Yuan-He/A3BL}}.
During the dataset construction process, we control the sample size by re-sampling data until the sequence size exceeds a threshold, denoted as \texttt{sample\_size}.
For \cref{fig:learnability}, the \texttt{sample\_size} is set to \(\mathtt{30,000}\), while for ensemble experiments it is set to \(\mathtt{120,000}\); other values are specified in the respective plots.
The reasoning model employs abductive reasoning, implemented using a cache-based search program~\citep{ABLkit2024}.

For an example of addition, the knowledge base is programmed as follows:
\begin{figure}[!htb]
  \centering
  \begin{lstlisting}[language=Python]
    class add_KB(KBBase):
      ...
      def logic_forward(self, nums):
        nums1, nums2 = split_list(nums)
        return digits_to_number(nums1) + digits_to_number(nums2)
  \end{lstlisting}
  \vspace{-0.35in}
  \caption{Example of addition knowledge base with Python program form.}
\end{figure}

For the modular addition task, the knowledge base is more complex:
\begin{figure}[!htb]
  \centering
  \begin{lstlisting}[language=Python]
    class Mod_KB(KBBase):
      ...
      def logic_forward(self, lsts):
          nums1, nums2, mod = parse_nums_and_mods(lsts)
          nums1, nums2 = digits_to_number(nums1), digits_to_number(nums2)
          return (nums1 + nums2) % mod
  \end{lstlisting}
  \vspace{-0.35in}
  \caption{Example of modular addition knowledge base with Python program form.}
\end{figure}

\subsubsection{Implementation details}
For all experiments, the random seeds are set to $\{2023,2024,2025,2026,2027\}$ for repeating five times.
To ensure the clustering property depicted in the definition of the restricted hypothesis space, the learning models are pre-trained. For LeNet5, we use self-supervised methods, with the weights available in the supplementary materials. For ResNet50, we load the pre-trained weights from the official PyTorch library, named \texttt{ResNet50\_Weights.IMAGENET1K\_V2}, and replace the last linear layer with \(\texttt{Linear(2048, 10)}\).

\paragraph{Optimization configurations} All experiments use AdamW, a weight-decay variant of Adam~\citep{adam}, as the optimizer, with a learning rate of \(\mathtt{0.0015}\) and betas set to \((\mathtt{0.9, 0.99})\). The batch size is set to 256, and unless otherwise noted, the number of epochs is set to \(\mathtt{10}\). The loss function used for optimization is cross-entropy, with further details available in the support code.

\subsection{Additional Experiments}
The additional experiments include setups using raw datasets not covered in the main text, specifically the CIFAR-10 and SVHN cases. Additionally, beyond the empirical analysis on the surrogate \cref{eq: a3bl}, here we refer to this as \AAABL~\citep{a3bl}, we also include an analysis of ABL and PNL to ensure a comprehensive evaluation.

\renewcommand{\algorithmiccomment}[1]{\hfill{\color{gray}\scalebox{1.2}{$\triangleright$} #1}} 
\begin{algorithm}[!tb]
    \caption{DCSP Solution}
    \label{alg:DCSP}
    \begin{algorithmic}[1]
        \REQUIRE NeSy task \(\nesytask\) and training set \(\{(\x_i, y_i)\}_{i=1}^N\)
        \STATE \( \mathsf{V,D,C} \gets \{V_i\}_{i=1}^L, \{D_i = \Z\}_{i=1}^L, \{\}\) \algorithmiccomment{Initialize the CSP triple}
        \FOR{\(i = 1 \ldots N\)}
            \STATE \(\mathsf{C} \gets \mathsf{C}\, \cup\, \{\mathsf{V}(\x_i) \land \KB \models y_i\}\) \algorithmiccomment{Initialize the constraints}
        \ENDFOR
        \STATE \(\mathcal{S} \gets \textsc{SolveCSP}(\mathsf{V,D,C})\) \algorithmiccomment{Call the CSP solver}
        \STATE \(d \gets L - \textrm{Union}(\mathcal{S})\)
        \RETURN \(\mathcal{S}, d\)
    \end{algorithmic}
  \end{algorithm}
\subsubsection{Observations of the structure of DCSP solution space}
The configurations of the modular addition task and their ensembles are illustrated in \cref{fig: configurations}. These configurations were computed using \cref{alg:DCSP} with the open-source library Choco~\citep{Prud2022Choco}. Through the investigation of the modular addition task, we observe that:
\emph{the number of DCSP solutions is highly related to DCSP solution disagreement; however, this relationship is not monotonic.}
Specifically, even a small number of DCSP solutions can result in high disagreement, as observed in the modular addition task with base \(k = \mathtt{10}\).

\begin{figure}[!htb]
  \centering
  \includegraphics[width=0.45\linewidth]{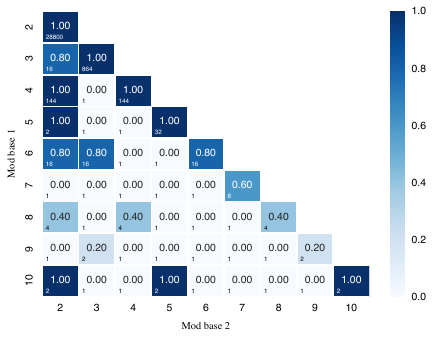}
  \caption{Configurations of modular additions and their ensembles. The center value represents the ratio of disagreement \(d\) to concept size \(L\), while the number of DCSP solutions is shown at the bottom-left.}
  \label{fig: configurations}
\end{figure}
\subsubsection{Impact of DCSP solution disagreement}
In \cref{app:fig:sample_complexity}, we present results under the same settings but using different raw datasets, specifically CIFAR-10 and SVHN, as shown in \cref{app:fig:sample_complexity}. As illustrated in the figure, the learnable case follows a trend similar to that in \cref{fig:sample_complexity}. However, in the unlearnable case, optimization becomes significantly more challenging due to high conflicts among valid DCSP solutions.
While one might suspect that this issue stems from the specific surrogate used, applying the same settings to ABL and PNL produces similar results~(cf. \cref{app:fig:sample_complexity_ABL} and \cref{app:fig:sample_complexity_PNL} respectively), confirming the generality of this observation.

\begin{figure*}[!tb]
  \centering
  \begin{subfigure}[b]{\linewidth}
    \includegraphics[width=\linewidth]{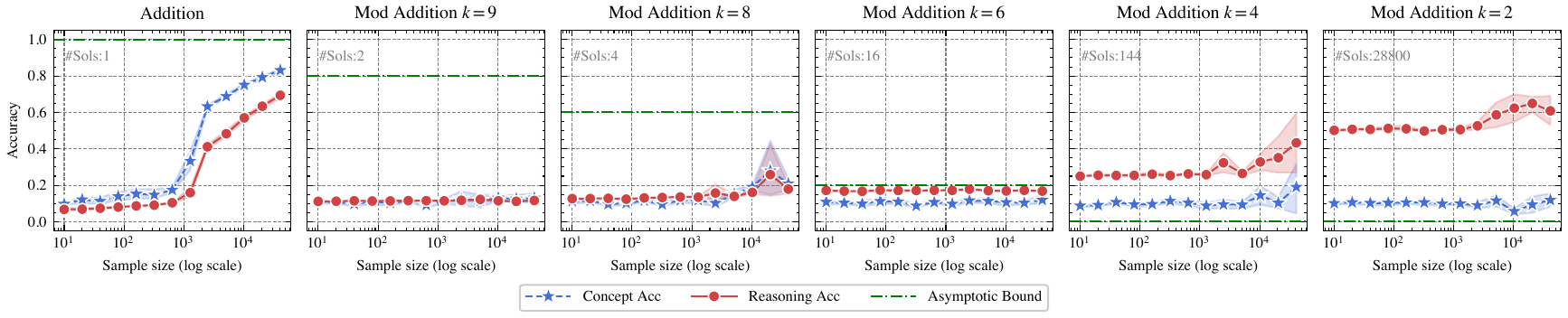}
    \caption{CIFAR-10}
  \end{subfigure}
  \begin{subfigure}[b]{\linewidth}
    \includegraphics[width=\linewidth]{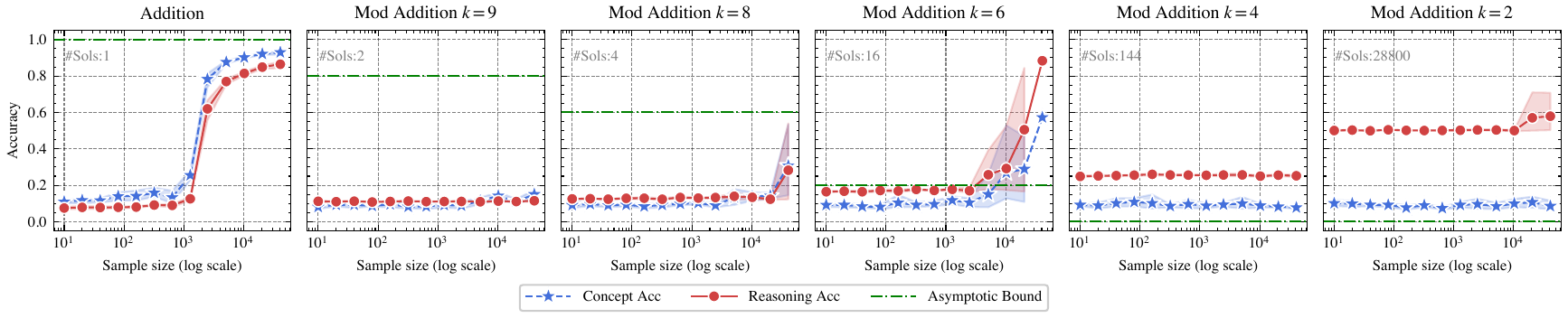}
    \caption{SVHN}
  \end{subfigure}
  \caption{\emph{Accuracies} versus \emph{sample size} for different NeSy tasks of \AAABL{}.
    The shadowed area denotes the standard error.
    The number of the DCSP solutions~({\color{gray}\#Sols}) is shown at the top left of each plot.
  The asymptotic bound~({\color{sgreen}green} line) from \cref{thm: average error bound} indicates that concept accuracy should exceed this bound as the sample size grows.}
  \label{app:fig:sample_complexity}
\end{figure*}

\begin{figure*}[!tb]
  \centering
  \begin{subfigure}[b]{\linewidth}
    \includegraphics[width=\linewidth]{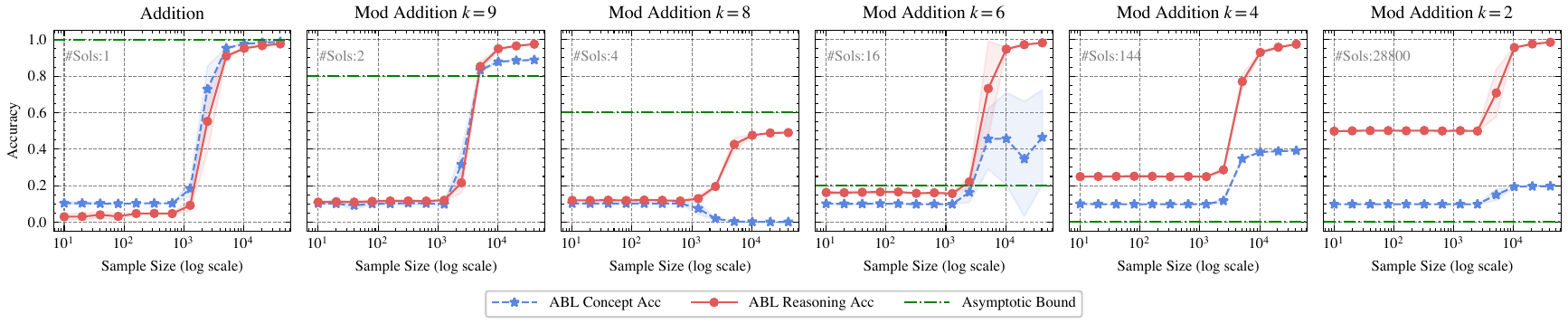}
    \caption{MNIST}
  \end{subfigure}
  \begin{subfigure}[b]{\linewidth}
    \includegraphics[width=\linewidth]{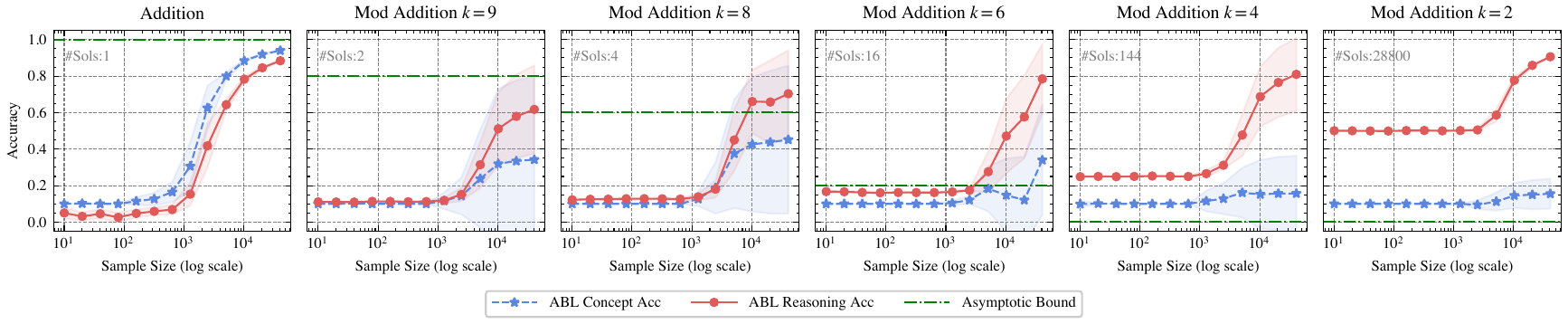}
    \caption{KMNIST}
  \end{subfigure}
  \begin{subfigure}[b]{\linewidth}
    \includegraphics[width=\linewidth]{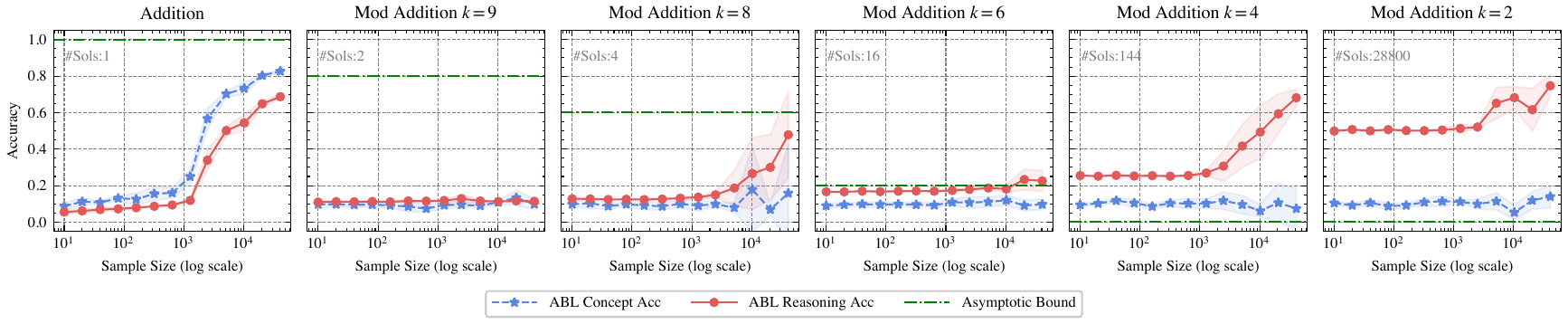}
    \caption{CIFAR-10}
  \end{subfigure}
  \begin{subfigure}[b]{\linewidth}
    \includegraphics[width=\linewidth]{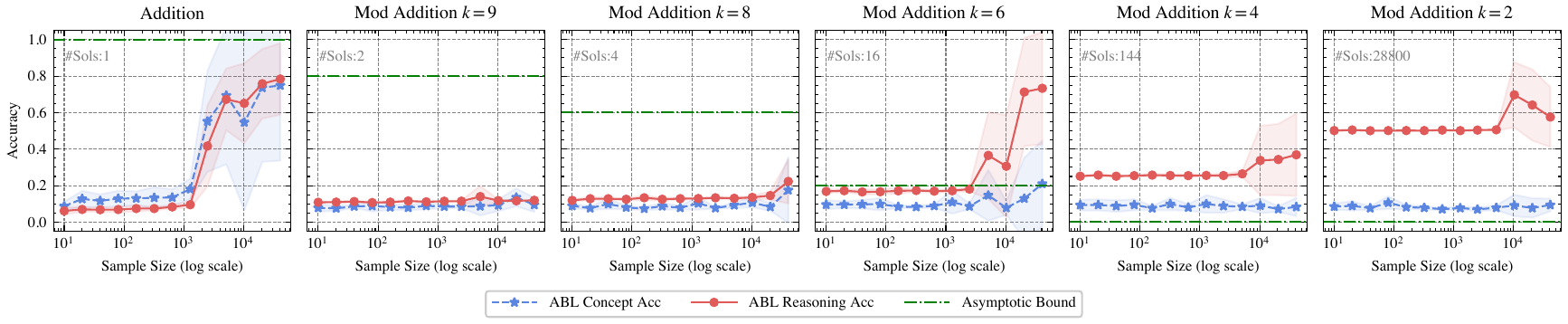}
    \caption{SVHN}
  \end{subfigure}
  \caption{\emph{Accuracies} versus \emph{sample size} for different NeSy tasks of ABL.
    The shadowed area denotes the standard error.
    The number of the DCSP solutions~({\color{gray}\#Sols}) is shown at the top left of each plot.
  The asymptotic bound~({\color{sgreen}green} line) from \cref{thm: average error bound} indicates that concept accuracy should exceed this bound as the sample size grows.}
  \label{app:fig:sample_complexity_ABL}
\end{figure*}

\begin{figure*}[!tb]
  \centering
  \begin{subfigure}[b]{\linewidth}
    \includegraphics[width=\linewidth]{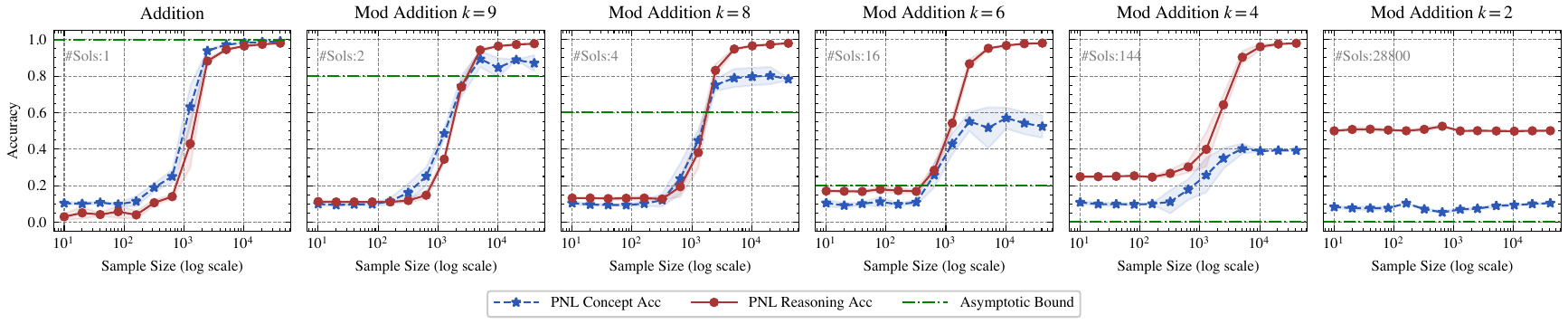}
    \caption{MNIST}
  \end{subfigure}
  \begin{subfigure}[b]{\linewidth}
    \includegraphics[width=\linewidth]{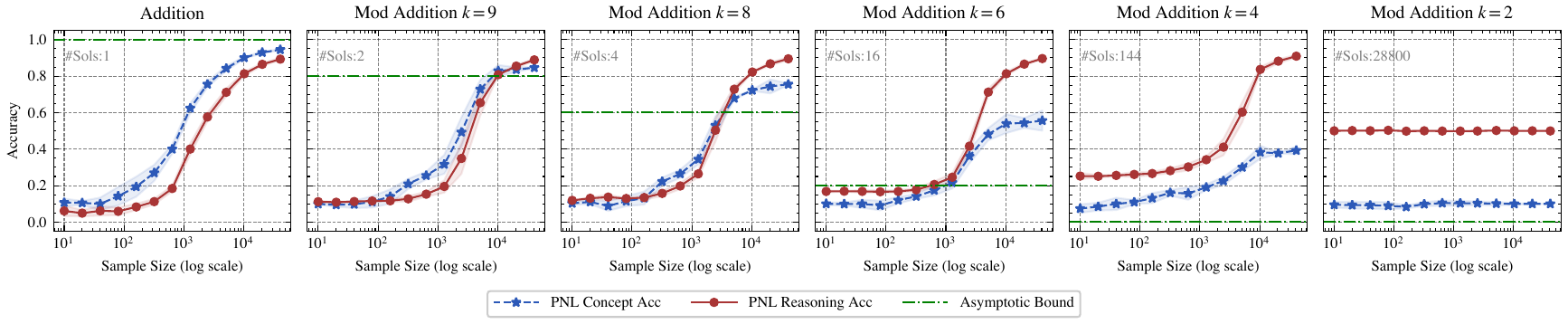}
    \caption{KMNIST}
  \end{subfigure}
  \begin{subfigure}[b]{\linewidth}
    \includegraphics[width=\linewidth]{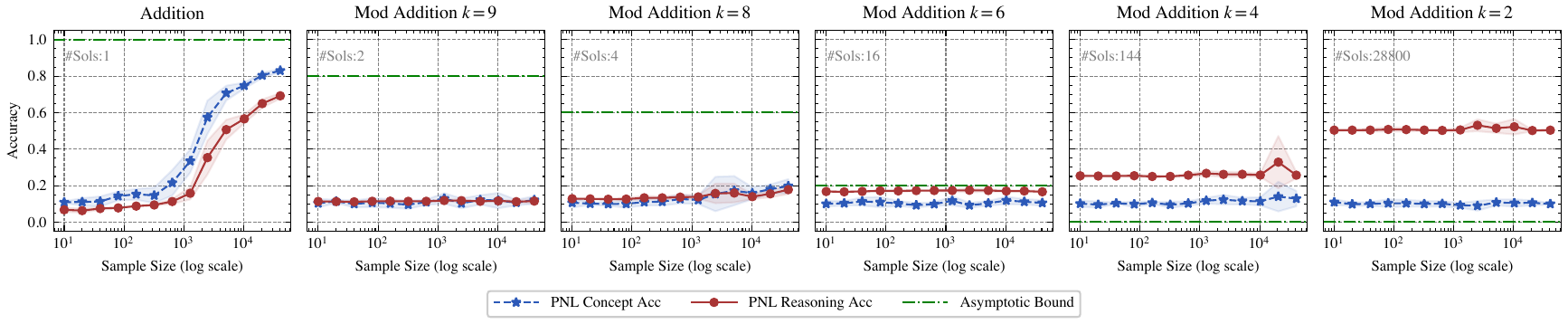}
    \caption{CIFAR-10}
  \end{subfigure}
  \begin{subfigure}[b]{\linewidth}
    \includegraphics[width=\linewidth]{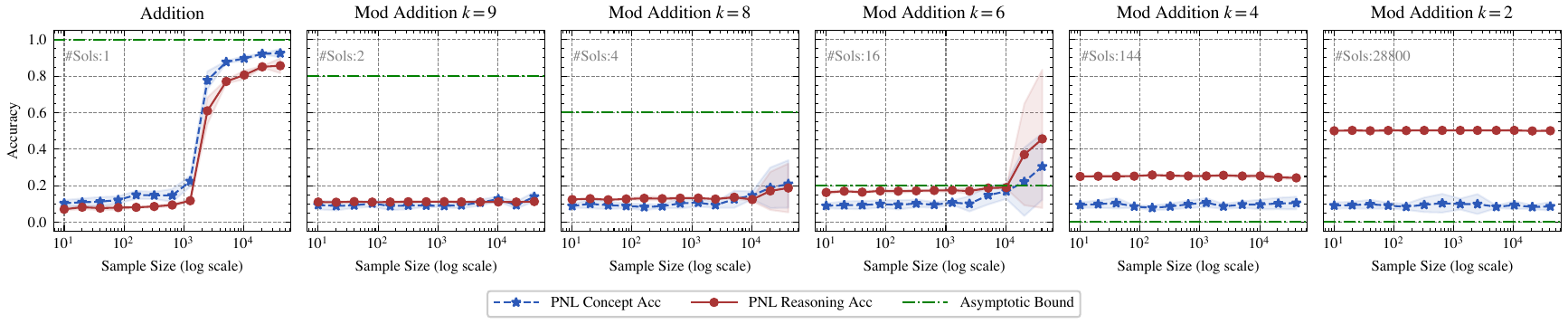}
    \caption{SVHN}
  \end{subfigure}
  \caption{\emph{Accuracies} versus \emph{sample size} for different NeSy tasks of PNL.
    The shadowed area denotes the standard error.
    The number of the DCSP solutions~({\color{gray}\#Sols}) is shown at the top left of each plot.
  The asymptotic bound~({\color{sgreen}green} line) from \cref{thm: average error bound} indicates that concept accuracy should exceed this bound as the sample size grows.}
  \label{app:fig:sample_complexity_PNL}
\end{figure*}

\clearpage

\subsubsection{Ensemble of unlearnable NeSy tasks}\label{app:details of ensemble experiments}
Here we present additional combinations of unlearnable NeSy tasks.
In \cref{app: fig: ensemble of unlearnable tasks}, we provide an overview of all ensemble combinations using a heatmap.
After that, we present a more fine-grained analysis.
In \cref{app: fig: ensemble of unlearnable tasks: unlearnable} and \cref{app: fig: ensemble of unlearnable tasks: learnable}, we show cases where the ensemble approach fails and succeeds respectively.

\begin{figure*}[!htb]
  \begin{subfigure}[b]{.4\linewidth}
    \centering
    \includegraphics[width=.8\linewidth]{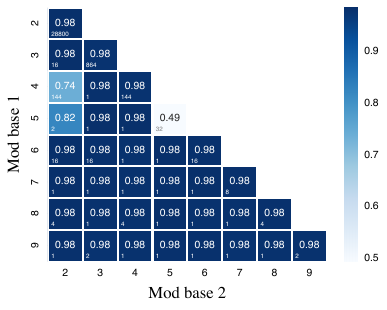}
  \end{subfigure} \hfill
  \begin{subfigure}[b]{.4\linewidth}
    \centering
    \includegraphics[width=.8\linewidth]{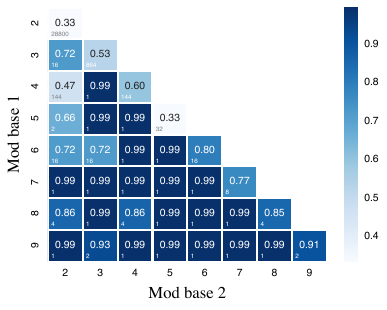}
  \end{subfigure}
  \vspace{-0.15in}
  \caption{\emph{Heatmaps} of ensemble mod addition tasks.
    Left: reasoning accuracy for different ensembles of mod bases; Right: concept accuracy for different ensembles of mod bases.
  The bottom-left corner of each cell shows the number of DCSP solutions.}
  \label{app: fig: ensemble of unlearnable tasks}
  \vspace{0.1in}
\end{figure*}

\begin{figure*}[!tb]
  \centering
  \begin{subfigure}[b]{\linewidth}
    \begin{subfigure}[b]{.495\linewidth}
      \centering
      \includegraphics[width=\linewidth]{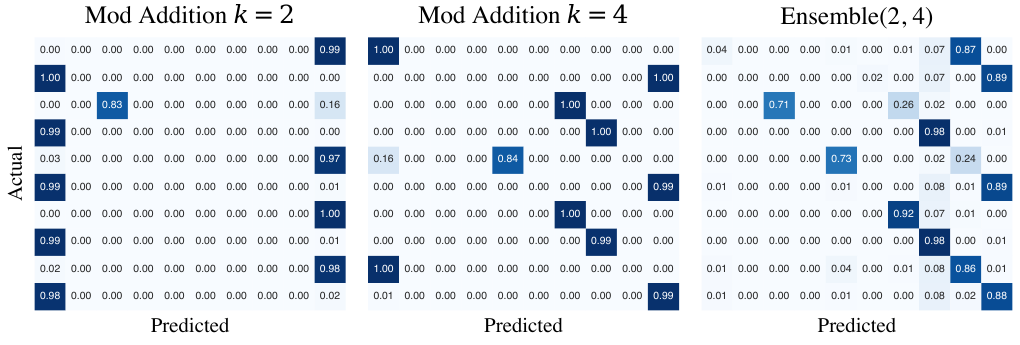}
    \end{subfigure} \hfill
    \begin{subfigure}[b]{.495\linewidth}
      \centering
      \includegraphics[width=\linewidth]{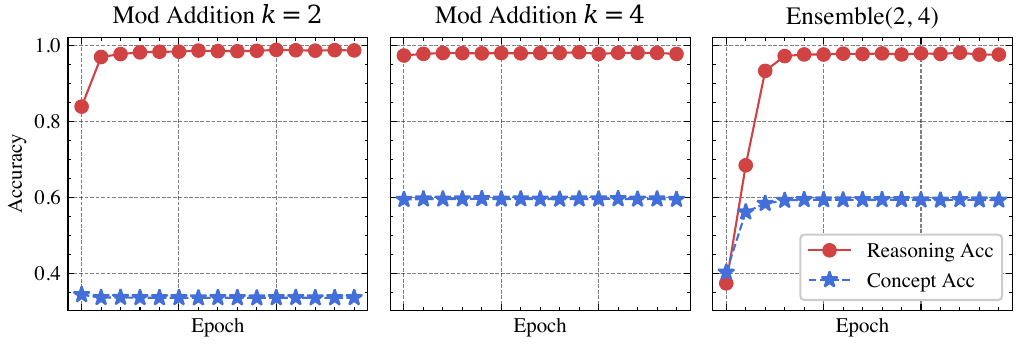}
    \end{subfigure}
  \end{subfigure}

  \begin{subfigure}[b]{\linewidth}
    \begin{subfigure}[b]{.495\linewidth}
      \centering
      \includegraphics[width=\linewidth]{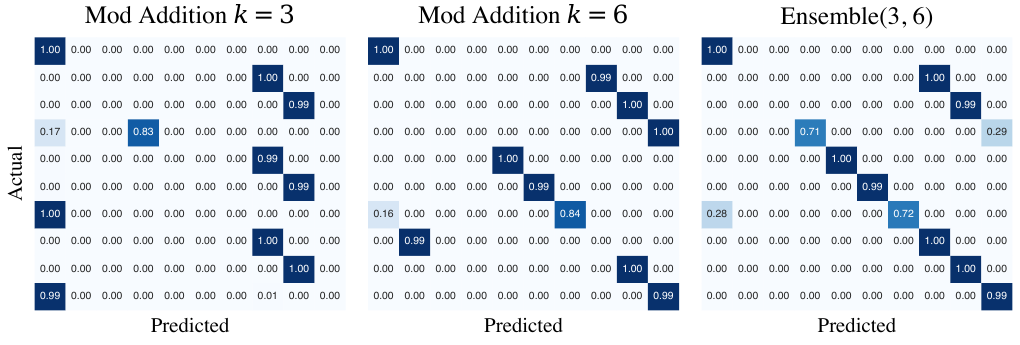}
    \end{subfigure} \hfill
    \begin{subfigure}[b]{.495\linewidth}
      \centering
      \includegraphics[width=\linewidth]{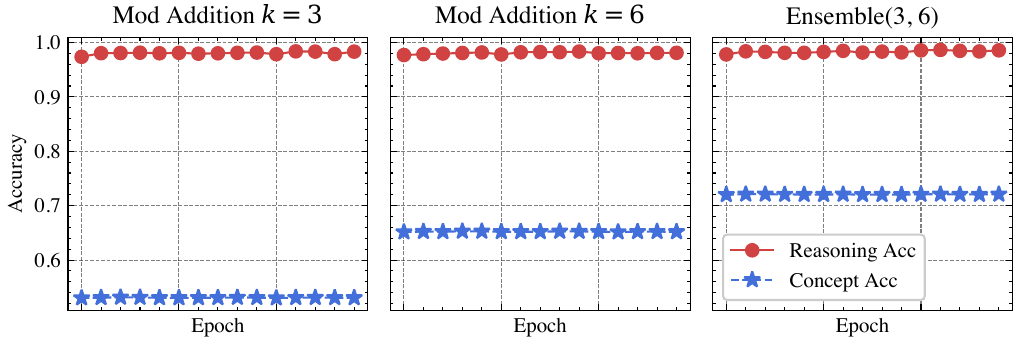}
    \end{subfigure}
  \end{subfigure}

  \begin{subfigure}[b]{\linewidth}
    \begin{subfigure}[b]{.495\linewidth}
      \centering
      \includegraphics[width=\linewidth]{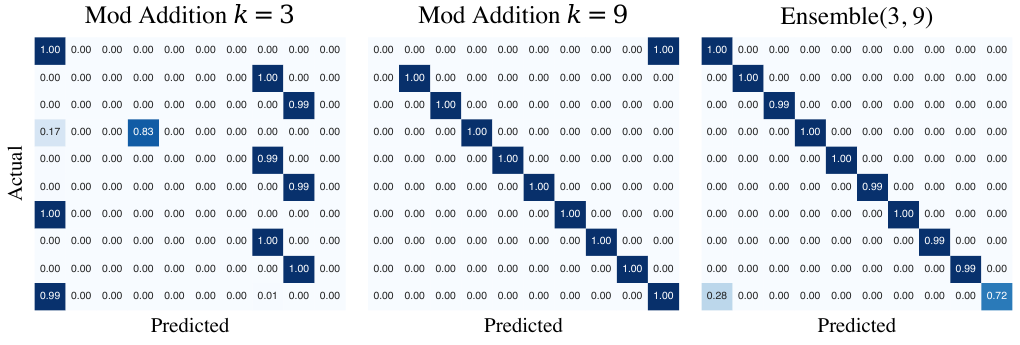}
    \end{subfigure} \hfill
    \begin{subfigure}[b]{.495\linewidth}
      \centering
      \includegraphics[width=\linewidth]{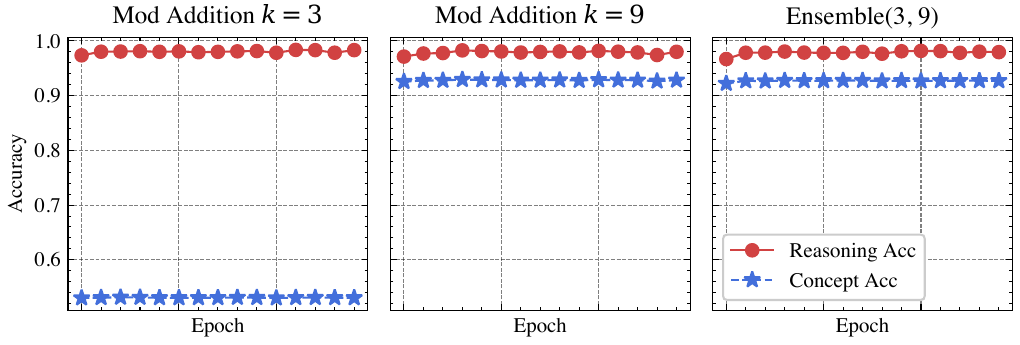}
    \end{subfigure}
  \end{subfigure}

  \begin{subfigure}[b]{\linewidth}
    \begin{subfigure}[b]{.495\linewidth}
      \centering
      \includegraphics[width=\linewidth]{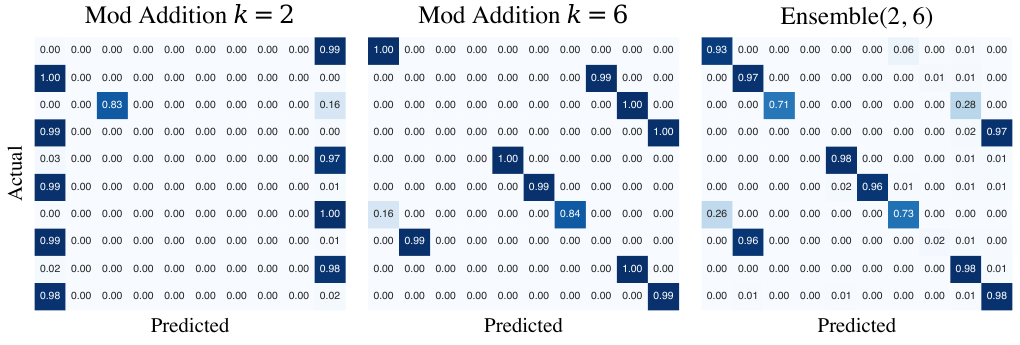}
    \end{subfigure} \hfill
    \begin{subfigure}[b]{.495\linewidth}
      \centering
      \includegraphics[width=\linewidth]{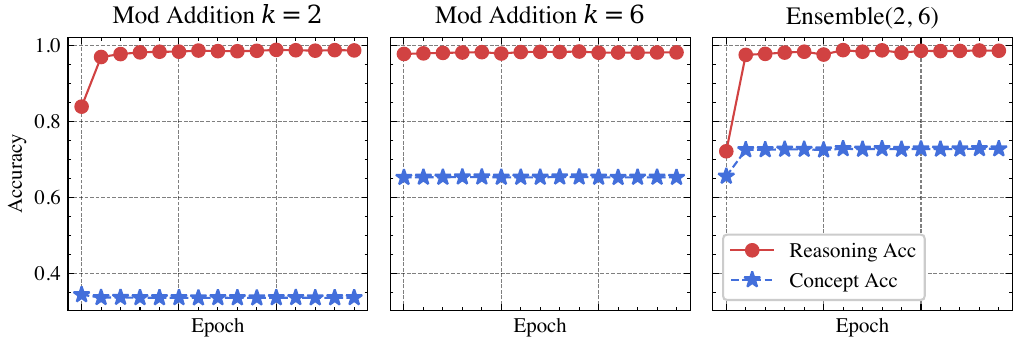}
    \end{subfigure}
  \end{subfigure}

  \caption{\emph{Ensemble} of unlearnable NeSy tasks, \emph{failed} case.
    The left shows confusion matrices and the right displays accuracy curves. 
    After the ensemble, the tasks are still unlearnable.
  }

  \label{app: fig: ensemble of unlearnable tasks: unlearnable}
\end{figure*}

\begin{figure*}[!tb]
  \begin{subfigure}[b]{\linewidth}
    \begin{subfigure}[b]{.495\linewidth}
      \centering
      \includegraphics[width=\linewidth]{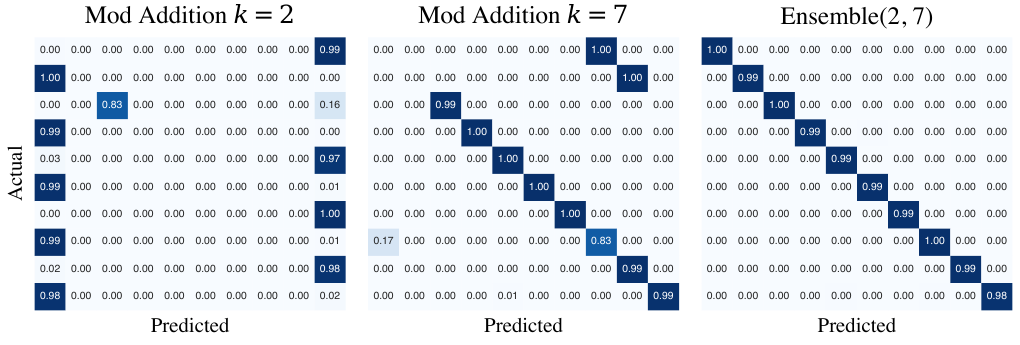}
    \end{subfigure} \hfill
    \begin{subfigure}[b]{.495\linewidth}
      \centering
      \includegraphics[width=\linewidth]{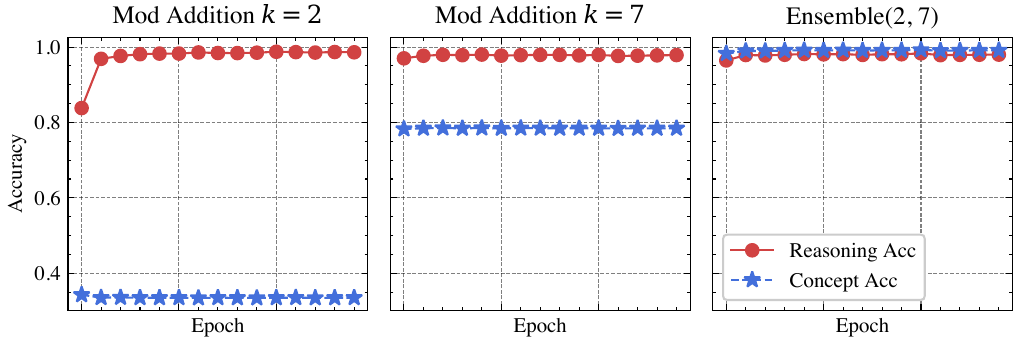}
    \end{subfigure}
  \end{subfigure}

  \begin{subfigure}[b]{\linewidth}
    \begin{subfigure}[b]{.495\linewidth}
      \centering
      \includegraphics[width=\linewidth]{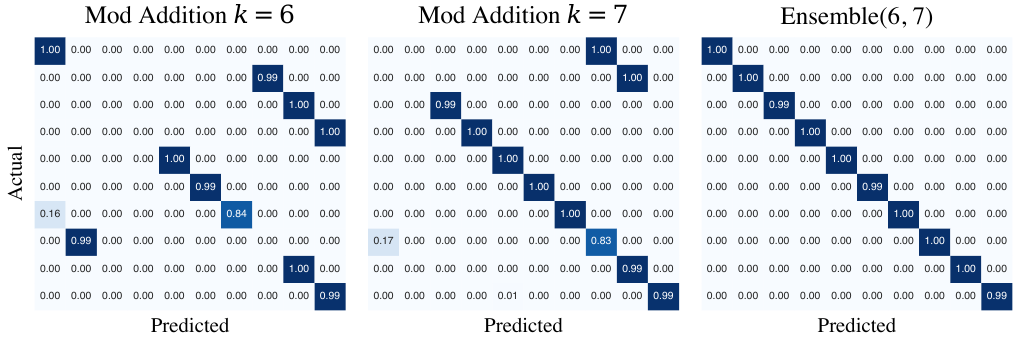}
    \end{subfigure} \hfill
    \begin{subfigure}[b]{.495\linewidth}
      \centering
      \includegraphics[width=\linewidth]{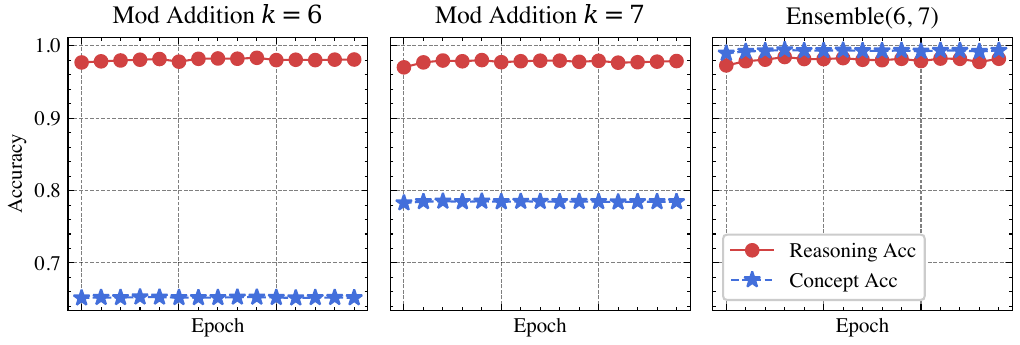}
    \end{subfigure}
  \end{subfigure}

  \begin{subfigure}[b]{\linewidth}
    \begin{subfigure}[b]{.495\linewidth}
      \centering
      \includegraphics[width=\linewidth]{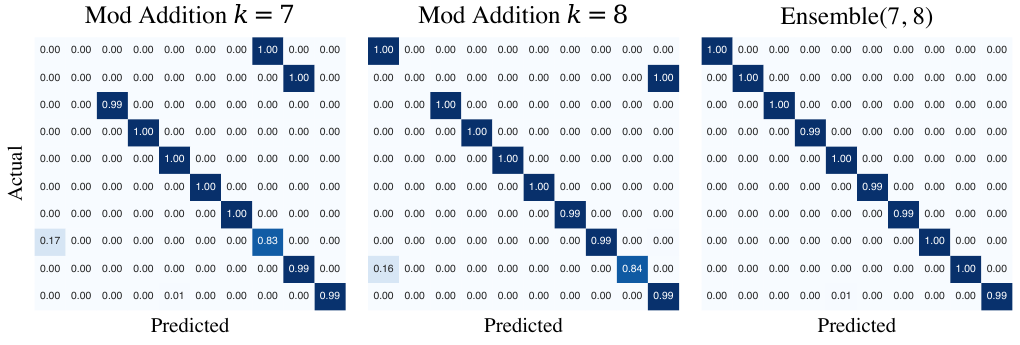}
    \end{subfigure} \hfill
    \begin{subfigure}[b]{.495\linewidth}
      \centering
      \includegraphics[width=\linewidth]{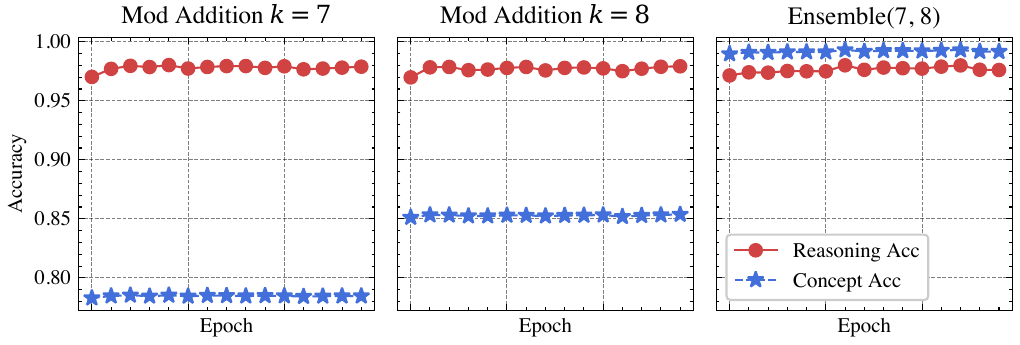}
    \end{subfigure}
  \end{subfigure}

  \begin{subfigure}[b]{\linewidth}
    \begin{subfigure}[b]{.495\linewidth}
      \centering
      \includegraphics[width=\linewidth]{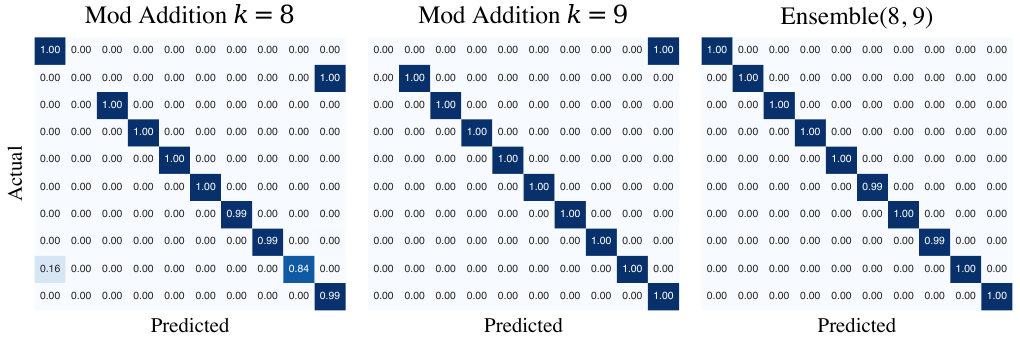}
    \end{subfigure} \hfill
    \begin{subfigure}[b]{.495\linewidth}
      \centering
      \includegraphics[width=\linewidth]{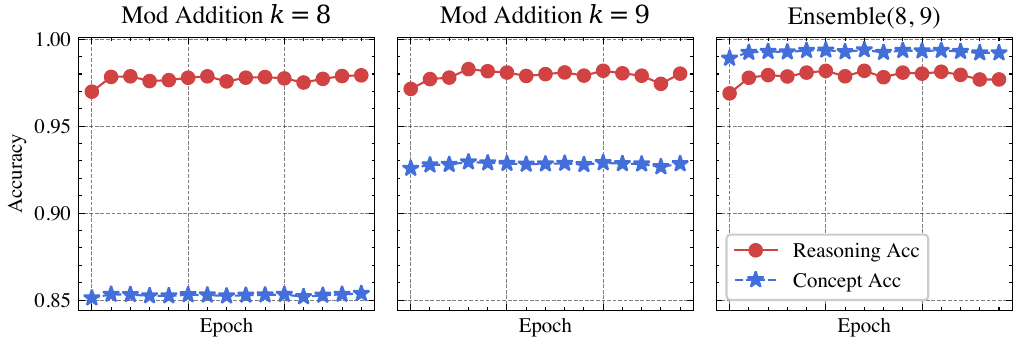}
    \end{subfigure}
  \end{subfigure}

  \caption{\emph{Ensemble} of unlearnable NeSy tasks, \emph{successed} case.
  The left shows confusion matrices and the right displays accuracy curves. 
  After the ensemble, the tasks become learnable.
}

  \label{app: fig: ensemble of unlearnable tasks: learnable}
\end{figure*}
By observing the experiments, we find that  adopting the ensemble perspective can enrich the benchmark diversity in the NeSy field~(e.g., \texttt{rsbench}, \citealt{bortolotti2024RSbench}), providing a clear and controllable methodology to achieve this.
\clearpage

\end{document}